\documentclass[11pt]{article}

\usepackage{jmlr2e02}
\input{macrogilles_main.tex}

\usepackage{algorithm}
\usepackage{algpseudocode}
\usepackage{bbm}

\ShortHeadings{Restless dependent bandits with fading memory}{Zadorozhnyi and Blanchard and Carpentier}
\firstpageno{1}

\begin{document}

\title{Restless dependent bandits with fading memory}

\author{\name Oleksandr Zadorozhnyi \email zadorozh@uni-potsdam.de \\
       \addr Department of Statistics\\
       University of Potsdam\\
       Potsdam, 14469, Germany
       \AND
       \name Gilles Blanchard  \email blanchard@math.uni-potsdam.de \\
       \addr Department of Statistics\\
        University of Potsdam\\
       Potsdam, 14469, Germany
   	   \AND
   	   \name Alexandra Carpentier \email alexandra.carpentier@ovgu.de \\
   	   \addr Department of Mathematical Stochastics \\
   	  Otto von G\"uricke Universit\"at Magdeburg\\
   	   Magdeburg, 39106, Germany
}

\editor{}

\maketitle

\begin{abstract}
 	We study the stochastic multi-armed bandit problem in the case when the arm samples are dependent over time and generated from so-called weak $\cC$-mixing processes. We establish a $\cC-$Mix Improved UCB agorithm and provide both problem-dependent and independent regret analysis in two different scenarios. In the first, so-called fast-mixing scenario, we show that pseudo-regret enjoys the same upper bound (up to a factor) as for i.i.d. observations; whereas in the second, slow mixing scenario, we discover a surprising effect, that the regret upper bound is similar to the independent case, with an incremental {\em additive}
 term which does not depend on the number of arms. The analysis of slow mixing scenario is supported with a minmax lower bound, which (up to a $\log(T)$ factor) matches the obtained upper bound.
  

\end{abstract}

\begin{keywords}
  Multi-armed bandits, Online Learning, Learning from dependent data observations.
\end{keywords}

	\section{Introduction}
\label{sec:introduction}

For positive integers $K>0$ and $T>K$, we consider the $K-$armed stochastic bandit problem with $T$ rounds.
In each round
$t \in \{1,\ldots, T\}$, the learner
selects an action (``arm'') $I_{t} \in \{1,\ldots,K\}$ and observes only a (stochastic)
version of the payoff $X_{t}^{I_{t}}$ based on their decision $I_{t}$.
The goal is to minimize their regret of not selecting the arm maximizing the cumulative average payoff. 
The formal definition of this notion 
is specified and discussed in Section~\ref{sec:setting}. 

In the broad literature on the stochastic bandits \cite{Robbins:52,Auer:02,Bubeck:12}, it is common to assume that the outcomes of the arms are stochastically independent. This means that for each round $t$, the distribution of the outcome $X_{t}^{k}$ of arm $k$ is not influenced by the history $\{X_{s}^{k}, s< t\}$ of the previous outcomes. 
From the application perspective, however, many of the real-world problems where bandits find their use display an intristic dependence between the sequence of future outcomes and of the past realizations. For example, in the ad-placement problem,
whether a user clicks on a given ad 
in the close future 
highly depends on whether he has clicked on the ad at the current time. The user can get bored, so that he will not be willing to choose this ad again immediately; however, as time passes by, the user is more likely to click on the ad once more. In such a scenario the correlation between previous observations and future ones decays as the time gap increase.
Additionally, in this example the decision to ultimately make a purchase might not be affected by the present minute fluctuations of the user's interest and only depends on its {\em nominal average} level, for which the long-term averaged clicking rate is a proxy. Another motivating example is attention detection. Here a human observer is looking at a screen, whereas the automatized monitor aims to identify the zones where the user is not paying sufficient attention. In order to do so, the monitor is allowed at each time $t$ to flash a small zone at the screen, e.g. light a pixel (action), and the eye tracker detects through the eye movement if the user has observed this flash (data point). In this example, the overall attention of the user fluctuates with time. During a short period of time the gaze is concentrated on the the same region of the picture, whereas, as the time evolves, it will migrate to different zones.
In this paper, we consider the framework where noisy rewards are generated from a (restless) weak $\cC$-mixing processes (see for example \citealp{Dedecker:15,Rio:00,Wintenberger:10}). This notion is used in order to describe the phenomenon of fading correlation between past and future of the stochastic process. Its precise definition will be given in the Section \ref{sec:setting}.  


Generalization of the stochastic bandits to the setting with dependent outcomes  
was considered by \citet{Whittle:88}. When the underlying stochastic processes are Markov Chains (a particular case of ${\cC}-$mixing process) with known dynamics, the regret was studied by \citet{Guha:10}, and
\citet{Ortner:14}. Problem-dependent asymptotic pseudo-regret upper bounds for the rewards generated from so-called $\varphi-$mixing processes were derived by \citet{Audiffren:15}. There,
authors devise an UCB-type strategy and consider scenarios of both fast and slowly mixing arms. Under the same weak dependency assumption of $\varphi$-mixing, the work of \citet{Gruene:17} extends the analysis to the different regret concepts, providing an upper bound analysis in the fast $\varphi$-mixing setting.

In the present work, we consider the more general notion of weakly dependent processes, which as a particular case includes $\varphi-$mixing. Using the probabilistic toolbox of concentration inequalities, developed in \citet{Deschamps:06}, we establish an algorithm ($\cC-$Mix Improved UCB) and obtain problem-dependent and -independent upper bounds on the pseudo-regret, both in the fast and slow mixing scenario. It is notable to mention that, even in the slow mixing scenario (i.e. when the correlation between past and future of the process decrease as a negative power with exponent less than 1), the obtained upper bounds remain close to the independent case. The contamination term due to dependency in the regret upper bound comes {\em additively} and does not scale with the number of arms. For the problem-dependent upper bound, it only depends on the 
choice of the threshold error level in the bound and on the mixing rate.
Since the main regret term (similar in order to the i.i.d. case) comprises a sum over arms of the inverse reward gaps, the contamination term can remain negligible by comparison, if there is a large number of suboptimal arms whose expected payoff is close to the chosen threshold level. This can be intuitively understood in that the time between two pulls of the same arm will typically remain larger than the correlation distance in that situation. Furthermore, for the problem-independent upper bound, the additive penalty due to dependency is determined by the relation between the number of arms, the exponent of the polynomially mixing process, and the time-horizon. In both cases, it allows to derive in the slow mixing scenario bounds which (in certain regimes) match their independent data analogues. 


The paper is organized as follows: in Section~\ref{sec:setting}, we introduce the concept of $\Phi_{\mathcal{C}}-$mixing process, the corresponding probabilistic toolbox and the notation
which will be used throughout the paper. In Section~\ref{sec:main_res}, we present the \textsc{$\cC-$Mix Improved UCB} learning algorithm and report the main bounds on its regret. Finally, in Section \ref{sec:discuss}, we discuss the obtained results and show their advantages in comparison to existing bounds, 
and illustrate situations where, even for highly correlated outputs, we can recover regret upper bounds which match in order the case of independent arm rewards. All proofs can be found in the supplementary material.

	\section{Setting and preliminaries}
\label{sec:setting}

Let  $\mathcal{X} = [0,1]$
 be the output space, equipped with the standard Borel $\sigma-$algebra $\cB$. Let $K$ be the number of arms and $T \in \mathbb{N}$ be the number of rounds (time-horizon). We always assume that $T > K$. We use the shortcut notation $[N]:=\{1,\ldots,N\}$ for any natural number $N$. For every arm $k \in [K]$, the outcomes are the samples from a stationary in time stochastic process  $\paren{X_{t}^{k}}_{t \in [T]}$ which is defined through its canonical version over the probability space $\paren{\Omega_{k} := \mathcal{X}^{T},\cB^{\otimes T},\mathbb{P}_{k}}$. For the process $X_{t}^{k}$ we denote also the 
canonical filtration $\mathcal{F}^{k}_{i} := \sigma\{X_{j}^{k}, j \leq i\}$. We assume each process to be weakly stationary and denote $\mu_{k} = \ee{}{X_{t}^{k}}$. 
We note also $\mu_{\star}:=\max_{k \in [K]}\mu_{k}$ for the arm with the highest average reward among all the arms in $[K]$, that we call the \textit{best arm}. 
When modelling the stochastic bandit problem  we consider the joint probability space $\paren{\Omega,\cA, \mbp}$, where $\Omega = \Omega_{1} \times \ldots \times \Omega_{K}$; $\mbp$ is a joint probability measure, whose $k-$th marginal is the measure $\mbp_{k}$ on $\Omega_{k}$, while $\cA$ is the product $\sigma-$algebra. 
We denote also $\cF_{i} := \sigma \{X^{[K]}_{j}: j \leq i \}$ to be the 
canonical filtration generated by \textit{all} stochastic processes.
Denote  $A^{'}:=\{k \in [K], \mu_{k}<\mu_{\star}\}$ the set of suboptimal arms and also write $\Delta_{k} := \mu_{\star}-\mu_{k}$ for the 
regret (in average) of playing suboptimal arm $k \in A^{'}$.  At each time step $t$, the learner chooses the action $I_{t} \in [K]$ and receives a noisy version of the reward $X_{t}^{I_{t}} \in \cX$. We consider the regret built up over $T$ rounds for the sequence of decisions $\paren{I_{t}}_{t \in [T]}$ and given as:
\begin{align}
\label{eq:pseudo_regret}
R(T) = \ee{}{\sum_{t=1}^{T} \mu_{\star} - \mu_{I_{t}}} =  T\mu_{\star} - \sum_{t=1}^{T}\ee{}{\mu_{I_{t}}} = 
\sum_{k=1}^{K}\Delta_{k}\ee{}{N_{k}\paren{T}}
\end{align}
where for every $k \in [K]$, $N_{k}\paren{T} = \sum_{t=1}^{T}\ind{I_{t}=k}$ denotes the number of times the learner chooses the arm $k$. In bandit literature, $R(T)$ is usually referred to as \textit{pseudo-regret}. 
In the following section we motivate the choice of this type of regret in our work. 

\subsection{Different notions of regret}
In the classical i.i.d. case, the pseudo-regret
$R(T)$defined in Equation~\eqref{eq:pseudo_regret} coincides
with the more standard notion of \textit{expected regret} 
$\overline{R}\paren{T}:= \sum_{t=1}^{T} (\mu_\star - \e[1]{X_{t}^{I_{t}}}),$ and furthermore $\mu_{*}$ is an upper bound on the
expected reward of {\em any} strategy (obviously attained for the ``oracle'' strategy always pulling the best arm).
We stress that {\em neither} of these facts hold in our setting with dependent arms.
Namely, since {\em both} the strategy choice $\{I_t = k\}$ at time $t$ and the arm outcomes $X_t^k$ depend on the past, we have $\e[0]{X_t^{I_t}}\neq\e{\mu_{I_t}}$. Furthermore, there might exist (oracle) ``arm switching'' strategies exploiting dependencies that
have significantly higher expected rewards\footnote{This second issue vanishes if fixed-arm strategies are the only admissible competitors for the regret, as is customary. 
} than $\mu_*$.
Pseudo-regret upper bounds for $\varphi-$mixing processes were studied previously
by 
\citet{Audiffren:15},
who, however, did not pointed out its relation to the \textit{expected} regret upper bounds. The last point was
discussed by~\citet{Gruene:17}, who argue that
the expected regret is the more natural notion if the observed outcomes
are direct rewards. They nevertheless analyze the pseudo-regret~\eqref{eq:pseudo_regret},
and provide approximation bounds covering both issues (Propositions~3 resp. 11 of \citealp{Gruene:17}): they bound the difference between $\mu_*$ and the expected reward of the best strategy, resp. the difference $\abs{R(T)-\overline{R}(T)}$. However, the first
bound is linear in $T$ and the second linear in $K$, so that the approximation bounds can become of larger order than the bound on
the pseudo-regret itself when $K$ and/or $T$ grow.

We argue here that pseudo-regret itself can more relevant in many situations of interest. In many applications, the observed outcome is not what one gets as a `true' reward later, but rather a proxy for it. One can model it as follows: at time $t$, the learner observes a sample from the mixing process $X_t^{I_t}$,  but the reward that she truly gets is $Y_t^{I_t}$ which is a stream generated independently from $X_t^{I_t}$, but has the same stationary distribution as $X_t^{I_t}$. The average regret in this setting is then exactly the pseudo-regret from Equation~\eqref{eq:pseudo_regret}.

For instance, in marketing applications, one often observes the fact that the customer clicks on a given item (given by process $X_t^{I_t}$), and considers it as a proxy for an
eventual buy of that same item (given by $Y_t^{I_t}$). But click probabilities are more volatile, and more subject to trends, than buy probabilities. Indeed, customers generally do not buy items immediately, and while they are likely to click on items that are trendy at a given time, it does not mean that they will buy it. Therefore in many cases, it makes more sense to aim at recommending items that have a high number of clicks in the long run, rather than aiming at maximization of the number of immediate clicks, which fluctuates more between items due to trends. In these settings, the
pseudo-regret is the relevant notion to look at; compared to the expected direct regret, it implies that the user is more interested in truly {\em learning}
about the long-run properties of the system, rather than in a notion of
immediate gratification (what you get later is more important that what you see now.)

Another setting of interest is that of {\em delayed}
rewards. First, a minor variation on the setting considered above is when
$Y_t^{I_t} = X_{t+\tau}^{I_t}$, where $\tau$ is a fixed
delay. In that situation, the delayed reward is still
different from the observation, but comes from the same
stream, and is therefore not independent, however for a very
large delay $\tau$ we have that
$\e[1]{Y_t^{I_t} | \cF_{t-1}} = \e{X_{t+\tau}^{I_t} | \cF_{t-1}}
\approx \mu^{I_t}$ for a mixing process with vanishing
dependencies over time (see next section). A different setting
is when the reward is indeed $X_t^{I_t}$, but is only observed after a delay $\tau$.
In other words, it is then required that the decision $I_t$ is $\cF_{t-\tau}$ measurable --
i.e. only based on past observations up to time $t-\tau$. In order to get to the main results,
we postpone a more detailed discussion of that setting in relation to our results to Section~\ref{se:delay}.



\subsection{Weak dependency assumption}
 
We posit a general type of weak dependency assumption that represent decay of correlation between past of the stochastic process and a moment in the future. We refer to \citet{Deschamps:06}, where such an assumption is given with respect to a class of real-valued functions. This notion generalizes the concept of weak-dependence for real-valued sequences (see for example \cite{Dedecker:15}) and, as it will be illustrated later, includes many important cases of stochastic processes.
Let $C(\cdot)$ be a semi-norm over a closed subspace $\mathcal{C}$ of the Banach space of bounded real-valued functions $f: \mathcal{X} \mapsto \mathbb{R}$ endowed with the norm: $\norm{f}_{\mathcal{C}} = C(f) + \norm{f}_{sup}$, where we denote as $\norm{f}_{sup}$ the standard supremum norm on $\mathcal{C}$. Denote $\mathcal{C}_{1}:= \{f \in \mathcal{C}: C(f) \leq 1\}$. Consider some random process $\paren{X_{t}}_{t \in [T]}$ defined over a probability space $\paren[1]{\cX^{ T},\cB^{\otimes T}, \mbp}$. 
\begin{definition}
	\label{def:phi_C_mixing}
	For $k \in \mathbb{N}$ define the $\Phi_{\mathcal{C}}-$ mixing coefficients as follows:
	\begin{align}
	\label{eq:mixing_coeff}
	&\Phi_{\mathcal{C}}(k) = \sup_{\varphi \in \mathcal{C}_{1}, i\geq 1} \big\{ \norm{E[\varphi (X_{i+k})|\mathcal{F}_{i}]-E[\varphi (X_{i+k})]}_{\infty} \big\},
	\end{align}
	where $\norm{\cdot}_{\infty}$ is the $L_{\infty}(\mbp)$ norm. We say that the process $\paren{X_{t}}_{t \in [T]}$ is ${\mathcal{C}}-$weak mixing (or simply $\mathcal{C}-$ mixing) with rate $\Phi_{\mathcal{C}}(k)$ if $\lim\limits_{k \rightarrow \infty} \Phi_{\mathcal{C}}(k) = 0$. Furthermore, we denote for the rate $\Phi_{\cC}\paren{k}$ through $\cM_{\Phi_{\cC}\paren{\cdot}}$ set of distributions $\mbp$ over $\paren{\cX^{ T},\cB^{\otimes T} }$ such that process $\paren{X_{t}}_{t\in [T]}$ is $\cC-$mixing with rate $\Phi_{\cC}\paren{t}$. 
\end{definition}
\begin{remark}
	We say that the $\mathcal{C}$-mixing process $\paren{X_{t}}_{t\in \mbn}$ is \textit{polynomially} mixing if, for large enough $t_{0}$, for all $t\geq t_{0}$, $\Phi_{\mathcal{C}}\paren{t} \leq c_{0}t^{-\alpha}$ with $\alpha >0 $ and $c_{0}$ some positive constant; and \textit{geometrically} mixing if, for large enough $t_{0}$, for all $t\geq t_{0}$, $\Phi_{\mathcal{C}}\paren{t} \leq c_{1}\exp(-t^{\gamma})$, where $c_{1},\gamma$ are some positive constants. 
\end{remark}
Fixing the seminorm $C(\cdot)$, we characterize the weak dependency assumption more precisely on particular examples.  
\begin{subsubsection}{Examples of weak- $\cC-$mixing sequences}
	It is straightforward to check that a process with independent
	outcomes is $\mathcal{C}-$mixing for any seminorm $C(\cdot)$ since $\Phi_{\mathcal{C}}\paren{k} = 0$, for all $k \geq 1$.
	
	If $C(\cdot)$ is taken to be the Lipschitz seminorm, i.e. $C_{Lip}(f) = \sup \big\{\frac{|f(s)-f(t)|}{\abs{s-t}} \big| s,t \in \mathcal{X}, s \neq t \big\}$,
	we obtain so-called $\tau-$mixing processes, see \citet{Dedecker:07,Wintenberger:10}. One can readily check that the auto-regressive process of order $1$ (AR-1), $ X_{i} = \rho X_{i-1} + \xi_{i}$, where $\xi_{i}$ is some bounded i.i.d. noise process, is geometrically $\tau-$mixing with rate $\Phi_{\mathcal{C}}(k) = \exp\paren[1]{-k \log{(\rho^{-1})}}$, provided $\rho<1$. 
	A moving-average process of finite order $q \in \mbn$, of the form 
	$$W_{i} = \mu + \sum_{j=0}^{q}\theta_{j}\psi_{i-j}, \text{for } i \in \mathbb{Z},$$ where $\paren{\psi_{i}}_{i \in \mathbb{Z}}$ is a sequence of bounded i.i.d. random variables and $\paren{\theta_j}_{0 \leq j \leq q}$ is
	a fixed vector in $\mbr^{q+1}$, can be shown to be geometrically $\tau-$mixing, provided certain assumption on the sequence $\paren{\theta_{j}}_{0 \leq j \leq q}$ holds (see for example in \citet{Canda:74}, also look in \citet{Rosenblatt:00} for a big overview on the mixing properties of linear processes). 
	
	Taking $C(f):= \norm{f}_{TV}$ to be the total variation norm on the bounded set $\mathcal{X}$, we obtain the so-called  $\tilde{\phi}$-mixing processes, described by \citet{Rio:96}. 
	Every recurrent aperiodic finite-state Markov chain  can be proved to be geometrically $\tilde{\phi}-$mixing with rate $\Phi_{\mathcal{C}}\paren{k} \leq \exp\paren{-k\log{\lambda^{-1}}}$, where $\lambda$ is the second largest eigenvalue of the transition matrix of the Markov chain. 
	Examples of polynomially $\tilde{\phi}-$mixing processes include several types of Metropolis-Hastings independent samplers in which the proposal distribution does not have a lower bounded density;  we refer the reader to \citet{Jarner:02} for this and further examples.

	\begin{remark}
		\label{rem:idn_mix}
		In the following we assume	that the identity function $\mbi: x \mapsto x$	belongs to the class $\mathcal{C}_{1}$ which implies that \begin{equation}
		\label{eq:mixingale}
		\norm{\ee{}{X_{i+k}|\cF_{i}} - \ee{}{X_{i+k}}}_{\infty} \leq \Phi_{\cC}\paren{k}
		\end{equation}
		  All of the aforementioned examples satisfy this assumption. 
		  Condition of boundedness  in \eqref{eq:mixingale} is also known as mixinagale condition (see \cite{Dedecker:07}).
	\end{remark}
\end{subsubsection}

\subsection{Concentration toolbox}
Our main technical toolbox for devising an UCB-type learning scheme is a general type of high probability maximal Hoeffding-type of concentration inequality which controls the deviations of the random sum for the stationary real-valued (mixingale-type) of stochastic process $\paren{X_{t}}_{t \in \mbn}$. The result is due to \cite{Peligrad:07} and we provide it below for completness. 

\begin{theorem}[Proposition 2 in \cite{Peligrad:07}] 
Let $\paren{Y_{t}}_{t \in \mbn}$ be a stationary real-valued centered process and define $S_{n}:= \sum_{i=1}^{n}Y_{i}$ as well as $S_{n}^{\star} = \max_{i\leq n}\abs{S_{n}}$. For $t\geq 0$, we have: 
\begin{align*}
\mbp\paren{S_{n}^{\star} \geq t} \leq 4\sqrt{e} \exp\paren{-t^2/2n\paren{ \norm{Y_{1}}_{\infty} + 80\delta_{n} }^2},
\end{align*}  
where $\delta_{n} = \sum_{j=1}^{n}j^{-\frac{3}{2}}\norm{\ee{}{S_{j}|\cF_{0}}}_{\infty}$ with $\cF_{0} = \sigma\paren{Y_{0}}$. 
Analogously, in the deviation form, we have that with probability at least $1-\delta$
\begin{align*}
S^{\star}_{n} \leq \sqrt{n}\paren{\norm{Y_{1}}_{\infty} + 80\delta_{n}}\sqrt{2\log{ \paren{\frac{\cA}{\delta}}}},
\end{align*}
where $\cA = 4\sqrt{e}$. 
\end{theorem}

Obviously, for a stationary $\cC-$weak mixing process $\paren{X_{t}}_{t \in [T]}$  the aforemenetioned Theorem can be applied by setting $Y_{t}:=X_{t}-\ee{}{X_{t}}$ and using the fact that it holds $\abs{S_{n}} \leq S^{\star}_{n}$. Using the definition of $\Phi_{\cC}-$mixing coefficients \eqref{def:phi_C_mixing} and consequence from Remark~\ref{rem:idn_mix} we obtain.

\begin{proposition}
	\label{prop:hoeff_bound_mix}
	For a stationary real-valued $\Phi_{\cC}-$mixing process $\paren{X_{t}}_{t \in \mbn}$ with rate $\Phi_{\cC}(t)$ we denote $S_{n} = \sum_{ t =1}^{n}X_{t}$ and $\mu =\ee{}{X_{t}}$. For any $\delta \in [0,1)$ with probability at least $1-\delta$ it holds: 
	\begin{align}
	\label{eq:hoeff_dep_gen}
			\abs{n^{-1}{S_{n}}- \mu} \leq \paren{1 + 80\sum_{j=1}^{n}j^{-\frac{3}{2}} \sum_{k=1}^{j}\Phi_{\cC}(k)}\sqrt{\frac{2\log\paren{\frac{\cA}{\delta}}}{n}}
	\end{align}
\end{proposition}

\begin{remark}
 Notice that the statement of Proposition~\ref{prop:hoeff_bound_mix} holds, when instead of $\paren{X_{t}}_{t \in \mbn}$, we consider $\paren{X_{kt}}_{t\in \mbn}$ for $k \in \mbn$ (i.e. sequence of random variables with gaps of fixed size). Namely, we have that with probability at least $1-\delta$ we have: 
 \begin{align}
 \label{eq:hoeff_gap}
		 \abs{n^{-1}\sum_{t=1}^{n}X_{kt} - \mu} \leq  
		 \paren{1 + 80\sum_{j=1}^{n}{j}^{-\frac{3}{2}} \sum_{\ell=1}^{j}\Phi_{\cC}(k\ell)}
		 \sqrt{\frac{2\log\paren{\frac{\cA}{\delta}}}{n}}
 \end{align}
 
 This result will be important to control the deviation of the estimate of the mean, obtained from the samples which are taken at given sequence of timepoints with constant gap in time.
\end{remark}
 
 \begin{remark}
 	\label{rem:fast_mixing}
 	Notice that the inequality  of Theorem~\ref{prop:hoeff_bound_mix} implies that there is the contamination factor due to dependence in the typical Hoeffding's concentration bound. However, in many cases of weak $\cC$-mixing, it enteres in the bound as a multiplicative constant. More precisely, consider $\Phi_{\cC}\paren{t} \leq t^{-\alpha}$ with some $\alpha >1/2$. Approximating the sum over mixing coefficients with the integral we get: 
 	\begin{align*}
	 	\sum_{j=1}^{n}j^{-\frac{3}{2}} \sum_{k=1}^{j}\Phi_{\cC}(k) \leq c_{\alpha}\sum_{j=1}^{n} j^{-\frac{1}{2}-\alpha},
 	\end{align*}
 	where $c_{\alpha}$ is some constant which depends on $\alpha$. Last partial sum is convergent for all $\alpha >\frac{1}{2}$. Therefore, from the Proposition~\ref{prop:hoeff_bound_mix} we deduce that for mixing process $\paren{X_{t}}_{t \in \mbn}$ with rate $\Phi_{\cC}\paren{t} \leq t^{-\alpha}$ with $\alpha > \frac{1}{2}$ and any $\delta >0$ with probability at least $1-\delta$ holds: 
 	\begin{align}
 	\label{eq:hoeff_fast_mix}
	 	\abs{n^{-1}{S_{n}}- \mu} \leq \paren{1 + M}\sqrt{\frac{2\log\paren{\frac{\cA}{\delta}}}{n}},
 	\end{align} 
 	where $M = 80c_{\alpha}\sum_{j=1}^{+\infty}j^{-\frac{1}{2}-\alpha} < \infty$ and $c_{\alpha}$ is a constant that depends only on $\alpha$. We refer to such type of weak dependence as to \textit{fast mixing scenario}.
 \end{remark}

	\section{Main Algorithm and regret upper bounds}
\label{sec:main_res}
The learning algorithm we present is conceptually based on the celebrated \textsc{Improved-UCB} learning algorithm  of \citet{Auer:10}. We also distinguish
an essential difference between the case where
all arms are polynomially mixing with exponent smaller than $1/2$ \textit{(slow $\mathcal{C}$-mixing)}, and the case where all arms are mixing sufficiently fast, typically polynomially mixing with exponent larger than $1/2$, or exponentially mixing,
so that the mixing coefficients for each arm are either summable or that partial sums of order $n$ diverge at speed not faster than $ \cO\paren{\sqrt{n}}$ \textit{(fast $\cC$-mixing)}. 
Essentially our \textsc{$\cC-$Mix Improved UCB} algorithm works as follows. Given the number of rounds $T$, the algorithm divides it into the epochs with nearly exponentially increasing number of pulls $t_{s}$ for every epoch $s$. At the given epoch $s$, the samples from each of the active arms are collected during the sequence of times with a constant gap $b_{s}$. This gap $b_{s}$ is equal to the number of active arms in the epoch $s$. At the end of the epoch the statistics for the mean and confidence of each arm are computed and the arms which perform poorly (in comparison to the empirically best arm) are eliminated (i.e. they are not considered to be active anymore). Then the algorithm proceeds to the next epoch. 	
Notice, that the start $\tau_s$ of epoch $s\geq 1$ is random and depends on the past observations. The sampling scheme of the epoch itself depends on the number of arms which were not eliminated during previous epochs (and is therefore random) but, given this information (mathematically represented as the $\sigma$-algebra $\cF_{\tau_{s}}$), the sampling scheme is deterministic. In such a learning scheme, to be able to use concentration inequalities for the estimation of the mean of each arm from the samples collected during the current epoch $s$, one needs to ensure that the process $\tilde{X}^{s}_{t} := \paren{X_{\tau_{s}+t}}_{t \in [T]}$ is $\cC-$ mixing {\em conditionally to $\cF_{\tau_s}$}, whenever the $X_{t}$ is $\cC-$mixing.
%
Note that in the basic i.i.d. setting this problem does not arise at all, since all characteristic properties of the independent process $\paren{X_{t}}_{t \in [T]}$ automatically carry over to the process $\tilde{X}^{s}_{t}$ conditioned to $\cF_{\tau_s}$. We justify below that this property indeed holds.	
\begin{proposition}
	\label{prop:cond_mix_prop}
	Consider a stochastic bandit problem in which each process is ${\cC}-$mixing with mixing rate $\Phi_{\cC}\paren{t}$.
	For any $k$, denote $\tau_{k}$ the random start of epoch $k$. 
	Let $\mathcal{F}_{\tau_{k}}$ denote the $\sigma-$algebra generated
	by the process $(X^{[K]}_t \mbf{1}\set{\tau_k \geq t})$, and $\prob{.|\cF_{\tau_k}}$ denote the regular conditional
	distribution of the observation process conditional to  $\cF_{\tau_k}$. Then, it holds $\mbp-$a.s. that
	the process $\tilde{X_t} = X_{\tau_s+t}$ is $\cC$-mixing with rate bounded by $2\Phi_{\cC}\paren{t}$
	under $\prob{.|\cF_{\tau_k}}$.
\end{proposition}

In the following, we assume that an a priori upper bound
on the mixing rate of all arms is known and set to be $\Phi_{\mathcal C}\paren{t}$;
this will determine the learning algorithm as well as the pseudo-regret
upper bounds. 

\begin{algorithm}[h!]
	\caption{\textsc{$\cC-$Mix Improved UCB}}
	\label{alg:cmix_ucb}
	\begin{algorithmic}
		\State {\bfseries Input:} Arms satisfy ${\cC}-$mixing assumption with  rate $\Phi_{\cC}\paren{t} \sim t^{-\alpha}$; set of arms $[K]$, the time-horizon $T>0$
		\State {\bfseries Initialize:} 
		 $\mathcal{A}= 4e^{1/2}$,$c_{0} = \paren{\paren{1-\alpha}\paren{1/2-\alpha}}^{-1}$, $c_{1} = \paren{\frac{\paren{1-\alpha}\paren{1/2 - \alpha}}{80}}^{\frac{2}{1-2\alpha}}$,
	$c_3 = 52400 c_{0}$, $s=0$, $\tau_{0}=1$ (starting time),$B_{0}:=[K]$ (number of arms), .
		\Repeat 
		\State 
		$\theta_{s} = 2^{-s}$,
		$t_{s}:= \paren{32c_1^{-1}\theta_{s}^{-2}\log\paren{\mathcal{A}T\theta_{s}^{2}}}^{\frac{1-2\alpha}{2\alpha}}$\\
		$b_{s}=\abs{B_{s}}$; \\
		\textbf{Select} the number of pulls $T_{s}$ as follows:	
		\begin{align*}
			T_{s} &= T_{s,1} := \bigg\lceil \frac{32 \log\paren{\mathcal{A}T\theta_{s}^{2}}}{\theta_{s}^{2 }}\bigg\rceil \text{ for } b_{s} \geq  t_s;\\
		T_{s}& = T_{s,2} := \Bigg\lceil \frac{1}{b_{s}} \paren{\frac{c_3 \log\paren{\mathcal{A}T\theta_{s}^{2}}}{\theta_{s}^{2 }}}^{\frac{1}{2\alpha}} \Bigg\rceil,\text{ for } b_{s} < t_{s}.
		\end{align*}
		\State 
		\For{$\ell \in \{0,\ldots,T_s-1\}$, $i \in B_s$} 
			\If{$\tau_{s}+i+\ell b_s >T$}
			\State break
			\Else
			\State choose the arm $i$ in $B_{s}$ at time point $\tau_{s}+i+\ell b_s$ 
			\EndIf
		\EndFor

		\State	\textbf{Arm elimination:} \\
		\textbf{Compute} $$\Omega\paren{\theta_{s},b_{s}} = 	 
		\paren{1 + 80\sum_{j=1}^{T_{s}}{j}^{-\frac{3}{2}} \sum_{\ell=1}^{j}\Phi_{\cC}(b_{s}\ell)}\sqrt{\frac{2\log\paren{\cA T \theta^{2}_{s}}}{T_{s}}}, $$ \\
		$$\hat{\mu}_{i,s} = T_{s}^{-1}\sum_{t=0}^{T_{s}-1}X^{i}_{\tau_{s} +i+ tb_{s}}$$ \\
		\textbf{Discard} all arms $i$  from $B_{s}$ for which: 
		\begin{align*}
		\hat{\mu}_{i,s}^{}+\Omega\paren{\theta_{s},b_{s}} \leq \max_{j \in B_{s}} 	\hat{\mu}_{j,s}^{} - \Omega\paren{\theta_{s},b_{s}}
		\end{align*} 
		\State $s = s +1$
		\State $\tau_{s} = \tau_{s-1} + b_{s}T_{s}$
		\Until{ 
			Horizont $T$ is reached}
	\end{algorithmic}
\end{algorithm}
\begin{remark}
	We remark that for our analysis it is sufficient to choose the upper bound on the last epoch $s_{\text{end}} := \lfloor \frac{1}{2} \log\paren{ \frac{\cA T}{32}}\rfloor$. Indeed, in Algorithm~\ref{alg:cmix_ucb} one can readily check that for all $s$ $T_{s,1}>T_{s,2}$ and furthermore $\log\paren{\cA T \theta_{s}^2} >1$ for all $s_{} \leq s_{\text{end}}$. Taking this into account, by plugging in $s_{\text{end}}$ into $T_{s,1}$ we obtain that at this epoch $T_{s_{\text{end}}} > T_{}$.  
\end{remark}
\subsection{Fast mixing scenario}
\label{subsec:fast_mix}
In the \textit{fast mixing} scenario, for which, as mentioned before $\Phi_{\cC}\paren{t} \leq t^{-\alpha}$ with $\alpha > \frac{1}{2}$
we make use of concentration inequality~\eqref{eq:hoeff_fast_mix} from Remark~\ref{rem:fast_mixing} in the  \textsc{Improved UCB} learning scheme, which was presented in \citet{Auer:10}. The latter considers sequential arm elimination during the epochs of increasing lengths, in which pulling sequences of the arms in each epoch can be arbitrary deterministic sequence.

Notice that by Proposition~\ref{prop:cond_mix_prop} the samples which are collected afresh from each new epoch, satisfy (conditionaly to the start of the epoch) the $\cC-$mixing property with rate $2\Phi_{\cC}\paren{t}$. Therefore, we can directly use  the concentration inequality~\eqref{eq:hoeff_fast_mix} replacing its i.i.d. counterpart as in the proof of Theorem 3.1 in \citet{Auer:10} for each given epoch $s$.  

%

We observe that, apart from multiplicative constant in the concentration inequality, there is no other influence of the contamination term on the concentration rate, therefore the analysis of \citet{Auer:10} can be repeated directly in the case of \textit{fast} ${\cC}-$mixing processes,
This gives the following problem dependent regret upper bound:
\begin{theorem}
	\label{thm:fast_mix_upp_bound}
	The pseudo-regret of the \textsc{Improved UCB} algorithm
	for the stochastic bandit problem
	in a \textit{fast  $\mathcal{C}-$mixing} bandit scenario is bounded by
	\begin{align*}
	R\paren{T} & \leq {\paren{1+M}} \sum_{k \in A_{\lambda}}  \paren{\Delta_{k} + \frac{96}{\Delta_{k}} +  { \frac{32\log \paren{T\Delta^{2}_{k}}}{\Delta^{}_{k}}}}  + 64\sum_{ k \in A_{0}\setminus A_{\lambda}}\frac{1}{\lambda} + \lambda T,
	\end{align*}
	where $M =80c_{\alpha}\sum_{j=1}^{+\infty}j^{-\frac{1}{2}-\alpha}$, $\lambda \geq \frac{e^{\frac{1}{4}}}{2\sqrt{T}}$ chosen arbitrary and $A_{\lambda} = \{k \in [K] \text{ s.t. } \Delta_{k} > \lambda \}$.
\end{theorem}
The following problem independent regret upper bound can be obtained from Theorem \ref{thm:fast_mix_upp_bound} using the same reasoning as \citet{Auer:10}. 
\begin{theorem}
	\label{thm:fast_mix_upp_ind}
	In the fast mixing scenario, the \textsc{Improved UCB} policy satisfies the following (problem independent) upper bound on {pseudo-regret } 
	\begin{align*}
	R\paren{T} \leq \sqrt{\paren{1+M}KT} \frac{\log\paren{K\log\paren{K}}}{\sqrt{\log\paren{K}}}.
	\end{align*}
\end{theorem}
\begin{remark}
	The upper bounds of Theorems~\ref{thm:fast_mix_upp_bound} and~\ref{thm:fast_mix_upp_ind}
	match (up to the multiplicative constant $\paren{1+M}$)
	the bounds on
	the {pseudo-regret} for Improved-UCB of \citet{Auer:10} 
	in the independent case.
\end{remark}
\subsection{Slow mixing scenario}
\label{subsec:slow_mix}
In this part we consider a more challenging slow mixing scenario. In this case the mixing rate of stochastic processes $X_{t}^{a}$ for $a \in [K]$ is assumed to be $\Phi_{\cC}\paren{t} \sim t^{-\alpha}$ with parameter $\alpha \in (0,1/2]$.
Theorem \ref{thm: upp_bound} provides the problem dependent upper bound for the pseudo-regret of the \textsc{$\cC-$MIX Improved UCB} learning strategy (Algorithm~\ref{alg:cmix_ucb} ) in the case of slow mixing scenario. The characteristic feature of the algorithm is that in each epoch, all remaining active arms are pulled cyclically, so that the time gap between two consequent pulls of one arm is equal to the number of active arms in the given epoch. This number is constant, given the observations until time $\tau_{s}$. Futhermore, with a slight difference to the original approach of \cite{Auer:10}, to estimate the mean in the slow mixing scenario during the epoch $s$ we consider samples collected \textit{during the time length of the epoch $s$} (and not during all the time). However, as the epoch's length increases geometrically with power larger then $2$, which means than more than the half samples are collected exactly at the epoch $s$, up to a multiplicative constant $2$ this provides the same effect for the estimation of the length of confidence term $\Omega\paren{\theta_s,b_s}$.
\begin{remark}
	We treat the case with $\alpha = \frac{1}{2}$ separetely. Notice, that from Remark~\ref{rem:fast_mixing} we deduce that although series $\sum_{j=1}^{\infty} j^{-\frac{1}{2}-\alpha}$ diverges, over the time horizont $T$ it makes the contribution of the term  of order $\log\paren{T}$. Thus, in this case we can apply the same scheme as in Theorem~\ref{thm:fast_mix_upp_bound}, provided that $M = \log\paren{T}$ and obtain the bounds of independent data scenario with the only contamination factor of order $\log\paren{T}$.  
\end{remark}
\begin{theorem}[Pseudo-regret upper bound for the \textit{slow mixing} scenario]
	\label{thm: upp_bound}
	Assume all arm processes $\paren{X_{t}^{k}}_{t \in [T]}$ are ${\mathcal{C}}-$polynomially mixing with the upper bound on the rate $\Phi_{\mathcal{C}}\paren{t} \leq t^{-\alpha}$, for $\alpha \in (0,1/2)$.
	Then the \textsc{$\mathcal{C}-$Mix Improved-UCB} satisfies the following
	{pseudo-regret} upper bound:
	\begin{align*}
	R\paren{T} & \leq  2 \sum_{k \in A_{\lambda}}  \max \{ c_{2}\Delta_{k}^{-1} \max\{\log{\paren{\mathcal{A}T\Delta_{k}^{2}} },1\},1  \} +  \tilde{c} \paren{{\Delta_{\star,\lambda}}}^{1-\frac{1}{\alpha}}\paren{{c_3 \log\paren{\mathcal{A}T\Delta_{\star,\lambda}^{2}}}}^{\frac{1}{2\alpha}} \\
	&\qquad + \frac{12}{\sqrt{e}} \sum_{ k \in A_{0}\setminus A_{\lambda}} \frac{1}{\lambda} +  \lambda T,
	\end{align*}
	where all numerical constants are defined as $\cA = {4\sqrt{e}}$, $c_{2}= 64c_{0}$, $c_{0} = \paren{(1-\alpha)(1/2-\alpha)}^{-1},c_{1} = \paren{\frac{\paren{1-\alpha}\paren{1/2-\alpha}}{80}}^{\frac{2}{1-2\alpha}}$, $c_{3}=12800c_{0}$, $c_{4}=\frac{1}{1.2 \sqrt{2.4}^{\frac{1}{\alpha}-2}-1}$, $\tilde{c} = 2^{-\frac{1}{\alpha}+3} c_4c^{\frac{1}{2\alpha}}_{3}$ and $\Delta_{\star,\lambda} = \min_{j \in A_{\lambda}} \Delta_{j}$, 
	while $\lambda \geq 0$ can be chosen arbitrary and $A_{\lambda} := \{k \in [K], \text{s.t. } \Delta_{k} > \lambda \}$.
\end{theorem}

\begin{remark}
	\label{rem:lambda_note}
	Notice that with the choice $\lambda \geq \sqrt{\frac{e^{1-1/e}}{T}}$ for $k \in A_{\lambda}$ we have $\Delta_{{k}}^{-1} \leq \lambda^{-1}$. Thus, since $\log\paren{\cA T\Delta_{k}^{2}} \leq \log\paren{T}$ and $\log\paren{T} >1$ we obtain the following Corollary in terms of the threshold $\lambda$ and the additive dependency term.
\end{remark}

\begin{corollary}
	\label{cor: upp_bound}
	For any choice of $\lambda$ which satisfies Remark~\ref{rem:lambda_note}, one has the following upper bound:
	\begin{align}
	R(T) &\leq \mathcal{O}\paren[3]{\sum_{k \in A^{'}} \frac{\log \paren{T}}{\lambda}
	}  + \mathcal{O}\paren{ \Delta_{\star, \lambda}^{\frac{\alpha-1}{\alpha}}\log^{\frac{1}{2\alpha}}(T)} + \lambda T, 			\label{eq:asympt_bound}
	\end{align}
	where $\Delta_{\star, \lambda}$ is defined as in Theorem \ref{thm: upp_bound}.
\end{corollary}
\begin{remark}
	\label{rem:probl_d}
	From the definition of $\Delta_{\star, \lambda}$, it follows that $\Delta_{\star, \lambda} ^{\frac{\alpha -1}{\alpha}} \leq \lambda^{\frac{\alpha -1}{\alpha}}$. This implies the following upper bound for the pseudo-regret in terms of the threshold $\lambda$: 
	\begin{align*}
	R\paren{T} \leq \frac{K\log\paren{T}}{\lambda} + \lambda^{\frac{\alpha-1}{\alpha}}\log^{\frac{1}{2\alpha}}\paren{T} + \lambda T.
	\end{align*}
	Furthermore, by straightforward comparison of the first two summands, for any admissible choice $\lambda$ from Theorem~\ref{thm: upp_bound}, if $\lambda \leq \paren{\frac{\log\paren{T}}{K^{\frac{\alpha}{1-2\alpha}}}}$  the term $\lambda^{\frac{\alpha-1}{\alpha}}\log^{\frac{1}{2\alpha}}\paren{T}$ dominates the other.
	Otherwise, if $\lambda > \paren{\frac{\log\paren{T}}{K^{\frac{\alpha}{1-2\alpha}}}}$, we have that $\frac{K \log\paren{T}}{\lambda}$ is of a larger order. 
\end{remark}
Analyzing the worst case scenario for the polynomially weak mixing processes, we obtain the following (problem independent) upper bound. 
\begin{theorem}[Problem-independent upper bound]
	\label{thm:prob_indep}
	Assume all arm reward processes are weak ${\mathcal{C}}-$polynomially mixing such that the conditions of Theorem~\ref{thm: upp_bound} hold. Then
	the \textsc{${\mathcal{C}}$-Mix Improved UCB} learning algorithm satisfies the following (instance independent) {pseudo-regret} bound: 
	\begin{align*}
	R\paren{T} \leq C_{3}\sqrt{T}\max \{\sqrt{K\log T}, T^{1/2-\alpha} \paren{\log T }^{\frac{1}{2\alpha}}\},
	\end{align*}
	where $C_{3}$ is some absolute numerical constant.
\end{theorem}
Notice that in the situation where $\alpha \rightarrow 0$ (correlations are very long-term), the regret bound scales almost linearly with the number of rounds $T$. 
%
\section{Lower bounds for regret in dependent bandit scenario}
It is natural to investigate the question whether the regret upper bounds we obtained in previous section are optimal, i.e. to search for lower bounds on the pseudo-regret $R\paren{T}$. Due to a broader case of $\Phi_{\cC}-$mixing dependent arm's outcomes this question can be adressed in case of different scenarios. Firstly, recall that in the fast mixing scenario, upper bounds of Theorems \ref{thm:fast_mix_upp_bound} and \ref{thm:fast_mix_upp_ind} match the corresponding problem independent regrets bounds for stochastic i.i.d. bandtis. From the works of \cite{Bubeck:12}, \cite{Bubeck:09} it is well known that $\sup \inf R\paren{T} \geq c\sqrt{TK}$, where infimum is taken over all strategies, supremum over all \textit{stochastic independent} bandits and $c$ being some small numerical constant. It implies that in the broader stochastic fast-mixing bandits scenario (which trivially extends independent stochastic bandits) our regret bounds are optimal up to a $log\paren{T}$ factor. 
In case of problem-dependent lower bounds, it is known (see \cite{Auer:02}) that $UCB1$ type of strategy is optimal in stochastic independent bandit case. More precisely, it is known that for any $\epsilon >0$ there is no learning strategy such that it holds $R\paren{T} \leq \sum_{ k : \Delta_{k}>0} \frac{\log\paren{T}}{\paren{2+\epsilon}\Delta_{{k}}}$, uniformly over independent distributions of arms $\paren{X_{t}^{k}}$, $k \in [K]$, $t\in [T]$.

Therefore, to fill the existing gap, it is interesting to consider the problem of lower bounds for stochastic bandits when all admissible environments are slow-mixing. In the works of \cite{Audiffren:15}, \cite{Gruene:17} authors also analyze setting of dependent bandits, however the question of lower bounds was not discussed there. Below we provide the problem-independent lower bound which matches (up to a factor of order $\log^{\frac{1}{2\alpha}}\paren{T}$) the regret upper bound in the case of slow mixing scenario. 

\begin{theorem}{Problem independent lower bound}
	\label{thm:prob_ind_lb}
	
	Let $\sup$ represents supremum taken over all stochastic $\Phi_{\cC}-$mixing bandits with rate which satisfies $\Phi_{\cC}\paren{t} \leq t^{-\alpha}$, $0 \leq \alpha <  \frac{1}{2}$ and $\inf$ be the infimum taken over all learning strategies $\paren{I_{t}}_{t \leq T}$. Then the following (problem independent) lower bound holds: 
	\begin{align*}
			\inf \sup R\paren{T} \geq\frac{1}{80}T^{1-\alpha}
	\end{align*}
	
\end{theorem}
	\section{Discussion}
\label{sec:discuss}
\subsection{Scenarios with independence regime } 
As we mentioned before, the upper bounds for the {pseudo-regret} in the fast mixing
case of Theorems~\ref{thm:fast_mix_upp_bound}, \ref{thm:fast_mix_upp_ind}  match the analogue results in the i.i.d. data scenario, up to a multiplicative absolute constant. Even in the cases, when series $\sum_{t=1}^{\infty} \Phi_{\cC}\paren{t}$ diverges, the influence of the penalization term due to dependence can be bounded by a constant, assuming polynomial rate with $\frac{1}{2} < \alpha$. Moreover, the independence regime regret upper bound can be recovered
even under slow mixing scenario (i.e. with $0 < \alpha \leq 1/2$). Namely, from the proof of Theorem~\ref{thm:prob_indep}, it follows that in the case  $K>T^{1-2\alpha}$, the main contribution to the regret upper bound in Corollary~\ref{cor: upp_bound} is given by the first term, which matches the well-known upper bound in the independent data scenario for \textsc{UCB}-type algorithms (see for example \citealp{Bubeck:12}; also in \citealp{Auer:02}). 
Also, for $\alpha \rightarrow 1/2$ we recover the optimal problem-independent bound (up to a square root of $\log\paren{T}$), which is typical for the all UCB type algorithms in the independent data scenario. In the special case $\alpha = 1/2$ we will have the typical UCB bound, contaminated by a multiplicative term $log\paren{T}$ (due to the influence of the sum of dependent coefficients). 

\subsection{Comparison to known regret upper bounds }
The existing literature on the regret analysis for the problem of stochastic bandits with dependent reward observations is relatively scarce. In the work of \citet{Audiffren:15} authors consider a type of weak-dependent process which is called $\varphi-$mixing. In the same setting of slow mixing processes as considered here, they obtain an asymptotic regret upper bound, which is of order $\widetilde{\Theta} \paren[1]{\Delta_{\star,\lambda}^{\frac{\alpha-2}{\alpha}} \log^{\frac{1}{\alpha}}\paren{T}}$, where the notation $\widetilde{\Theta} \paren{f}=g$ means that there exists $\gamma,\beta > 0$ so that $\abs{f} \log^{\beta}\paren{\abs{f}} \leq \abs{g}$ and $\abs{g} \log^{\gamma}\paren{\abs{g}} \leq \abs{f}$. 

Comparing our result to the bound of \citet{Audiffren:15}, we observe an improvement in regret upper bounds in several regards. First, our upper bound is not asymptotic in nature; also it depends explicitly on the gaps of \textit{all} suboptimal arms while \citet{Audiffren:15} provide an upper bound in terms of the worst gap only (which corresponds to the penalty term $\Delta_{\star, \lambda}$).  
Secondly, Corollary~\ref{cor: upp_bound} gives an asymptotic bound, which ( in the case when the penalty term $\Delta_{\star, \lambda}^{\frac{\alpha-1}{\alpha}}\log^{\frac{1}{2\alpha}}\paren{T}$ has the largest impact on the bound) is better by 
a polynomial factor in terms of the smallest gap and in the power of log-term of the time horizont. 
Also, the additive penalty term due to dependency from Theorem~\ref{thm: upp_bound} does not scale with the number of arms (which improves over the similar results of \citealp{Audiffren:15} for the particular case of $\varphi-$mixing processes). Lastly, our analysis is provided for a broader class of weak-dependent processes, which as its particular case include $\varphi-$mixing.

Furthermore, comparing our pseudo-regret upper bounds to the result of \citet{Gruene:17} (Theorem 10 therein), we observe that in the fast mixing scenario the latter scales in the same way (namely with a constant factor which depends on the sum of mixing coefficients) and has the same order of magnitude in $\Delta_{{k}}$ and $T$. However, our work contributes to the analysis of the slow mixing scenario ( and for a much general class of processes), which was not covered in \citet{Gruene:17} and provides the matcing (up to log terms) upper bound, showing optimality in slow mixing scenario. 
\begin{remark}
	The pseudo-regret upper bounds from Theorem \ref{thm: upp_bound} cannot be obtained by simply using standard Improved-UCB algorithm while plugging in variations of Hoeffding's concentration inequalities with "worse" deviation rates ( see e.g. \citealp{Lugosi:13} for heavy-tailed bandits). 
	With such an approach, one gets a penalty term with the worse rate inside {the sum over suboptimal arms}, so that the pseudo-regret will scale linearly with the number of arms. 
	We insist that the surprising effect of additive contamination is specific to the weak-dependent scenario and the proposed strategy,
	exploiting the knowledge of mixing coefficients and the number of arms in each epoch. 
	For comparison, notice that we can still use the fast mixing results in
	the slow mixing scenario, since $T$ is finite and we can take there $M=\sum_{t=1}^{T}\Phi_{\cC}\paren{t}$ (now growing with $T$).
	Comparing the problem-independent upper bound of
	Theorem~\ref{thm:fast_mix_upp_ind} (standard Improved-UCB) to that of Theorem~\ref{thm:prob_indep} (Algorithm~\ref{alg:cmix_ucb}),
	we see that using Algorithm~\ref{alg:cmix_ucb} gives arguably better bounds.
	Namely, applying the standard Improved-UCB learning scheme in the slow mixing regime results in the regret upper bound in Theorem~\ref{thm:fast_mix_upp_ind} being impacted by a multiplicative scaling factor ${1+M} \sim {\sum_{t=1}^{T}\Phi_{\cC}\paren{t}} \sim T^{1/2-\alpha}$. This gives (disregarding the influence of logarithmic terms) a bound of order $T^{1-\alpha}\sqrt{K}$.
	This is worse than the bound of Theorem~\ref{thm:prob_indep}, which does not have the scaling factor in the number of arms
	in the corresponding term.  
\end{remark}

\subsection{Dependent setting with delays}
\label{se:delay}

We come back to another setting briefly discussed earlier, where the developed theory is of interest, namely
delay between dependent observations and actions. Let $\tau>0$ be some integer number.
Assume that due to various constraints, the learner makes the choice $I_t$ based not on the immediate history, but on the outcomes observed up until time $t-\tau$
This is a specific case of the so-called delayed bandit problem, which was first studied  for independent outcome observations by \citet{Pal:10}.  The effects of the delay in this setting results in an {\em additive} penalty (depending linearly on $\tau$) for the regret (see \citealp{Joulani:13}; and \citealp{Desautels:14} for the case of Gaussian rewards). If we consider the case of arbitrary data sequences (i.e. the adversarial bandit problem), \citet{Neu:13} showed a regret upper bound increased by a multiplicative factor
in $\tau$ with respect to the standard case (see also \citealp{Joulani:13} for the more general problem with random delay times). We exhibit here
the intermediate position of the random weakly dependent setting in the delayed feedback bandit problem.  

Formally, define an admissible $\tau-$delayed policy $I=\paren{I_{t}}_{t \geq 1}$ as a function taking values in $[K]$ as follows:
\begin{align}
\begin{aligned}
\label{eq:policy_defintion}
I_{t} = \begin{cases} \text{ choose arm randomly} &\mbox{if } t < M; \\ 
I_{t}\paren{X^{I_{t-\tau}}_{t-\tau},\ldots,X_{1}^{I_{1}}} & \mbox{if } t \geq M. \end{cases}
\end{aligned}
\end{align}
In other words, by putting $Y_{i} := X_{i}^{I_{i}}$ for $i \geq 1$ and defining $\mathcal{M}_{t} := \sigma\{Y_{1},I_{1},\ldots,Y_{t},I_{t}\}$ we assume $I_{t}$ to be $\mathcal{M}_{t-\tau}$ measurable.
We now show that in the delayed feedback setting, we the pseudo-regret is a good approximation of the expected regret
if the delay is large. Namely, consider the expectation of the sample rewards $\e[1]{\sum_{t=1}^{T}X_{t}^{I_{t}}}$.
By using the tower property and the definition of $\cC$-mixing,
we have: 
\begin{align*}
\ee{}{\sum_{t=1}^{T}X_{t}^{I_{t}}} &= \sum_{t=1}^{T}\sum_{k=1}^{K}\ee{}{X_{t}^{I_{t}}\mbi_{I_{t}=k}}\\
& \geq \sum_{t=1}^{T}\sum_{k=1}^{K} \ee{}{\paren{\mu_{I_{t}} - \Phi_{\mathcal{C}}(\tau)}\mbi_{I_{t}=k}}= \sum_{t=1}^{T}\ee{}{\mu_{I_{t}}} - \Phi_{\mathcal{C}}(\tau)T.
\end{align*}
By symmetry, we can apply the same reasoning and get the reverse bound. Uniting these contributions we get the two-sided control over $\e[1]{\sum_{t=1}^{T}X_{t}^{I_{t}}}$: 
\begin{align}
\label{eq:regdelayed}
\abs{\e[3]{\sum_{t=1}^{T}X_{t}^{I_{t}}} - \sum_{t=1}^{T}\e[2]{\mu_{I_{t}}}} \leq \Phi_{\mathcal{C}}\paren{\tau}T,
\end{align}
which, together with the definitions of expected regret and {pseudo-regret} implies that for any strategy $\pi_\tau$
(i.e. a choice of decision functions $I_t$ adapted to the $\tau$-delayed setting):
\begin{align}
\label{eq:delayapprox}
\abs{R_\tau(\pi_\tau,T) - \overline{R}_\tau\paren{\pi_\tau,T} } \leq \Phi_{\mathcal{C}}\paren{\tau}T,
\end{align}
where the index $\tau$ indicates the $\tau$-delayed setting.
Now, it is easy to see that any strategy $\pi_1$ designed for the standard setting can be used in the delayed setting from
time $t=\tau$ onwards and that we then have
\begin{align}
\label{eq:burnin}
R_\tau(\pi_1,T) \leq R_1(\pi_1,T-\tau) + \tau,
\end{align}
and we can apply directly the bounds concerning $R_1$ coming from Section~\ref{sec:main_res} with the
time horizon $T'=T-\tau$. The additive term above represents the worst-case regret for the $\tau$ steps of ``burn-in'', where
decisions must be taken blindly.

We discuss now in which regimes the approximation terms discussed above are of smaller order than the bounds on $R_1$ obtained via
our main results. We discuss the problem-independent bounds for simplicity. In the fast mixing case, the main bound is of order
$\mtc{O}(\sqrt{T})$, so the ``burn-in'' term will be of smaller order if $\tau = \mtc{O}(\sqrt{T})$. Also requiring
the regret approximation~\eqref{eq:delayapprox} to be $\mtc{O}(\sqrt{T})$ leads to $\Phi_{\mathcal{C}}\paren{\tau}\lesssim T^{-\frac{1}{2}}$, i.e. $\tau$ should be at least logarithmic in $T$ under exponential mixing, and $\tau \gtrsim T^{\frac{1}{2\alpha}}$
for power-type mixing $\Phi_{\cC}(t)=t^{-\alpha}$, $\alpha\geq 1/2$.

In the slow mixing case $\Phi_{\cC}(t)=t^{-\alpha}$, $\alpha < 1/2$, the main term in the regret bound is at least of order $T^{1-{\alpha}}$.  The $\tau-$delayed approximation bound of equation~\ref{eq:delayapprox} is of order $\tau^{-\alpha}T \geq T^{1-\alpha} $ (under natural condition that $\tau \leq T$) and thus of larger order than the contribution of any of the bounds to the regret.  
%

In the several regimes delineated above for the size of the delay with respect to the total time horizon $T$,
the approximation terms can be neglected and the true average regret $\overline{{R}}_\tau\paren{T}$
satisfies regret bounds of the same order as the upper bounds obtained on the pseudo-regret in the standard scenario.
These results are to be compared with existing ones in the delayed observation setting, see for example
\citet{Joulani:13}, Table~1. In the stochastic and independent case, the burn-in inequality~\eqref{eq:burnin} shows that
the same bounds as in the delay-less case apply up to a linear term in the delay. On the other hand, in the adversarial
setting, existing bounds show $R_\tau(T) \lesssim \tau R_1(T/\tau)$ \citep{Neu:13} resulting in a $\sqrt{\tau}$ factor
increase (for problem-independent bounds) as compared to the delay-less case. It is striking that the relatively narrow gap
between regret bounds in the stochastic independent setting  and the adversarial setting widens in the delayed setting when considering depependent observations.
Moreover, the regret in the stochastic weakly dependent setting can, even when the correlations are decaying slow (i.e. $\frac{1}{2} \leq \alpha \leq 1$), remain close to the regret bound in the independent setting,
or  (in case when of slow mixing scenario, when $\alpha < \frac{1}{2}$) an intermediate position in between two regimes.

\section{Conclusions}
\label{sec:conclusions}
We have studied an extension of the stochastic bandit problem to the case where the arm processes satisfy a weak-dependency assumption of general kind. It characterizes the decay of correlations between past and future of the process and is measured by $\Phi_{\cC}-$mixing coefficients. Through the \textsc{C-MiX-Improved} UCB Algorithm we recover in many scenarios (i.e. in the ``fast mixing'' case where the mixing coefficients have either exponential or polynomial decay with power $\alpha>\frac{1}{2}$) we recover up to an absolute  constant the same
pseudo-regret upper bounds  as in the independent data scenario.

Furthermore, even in the case when the processes are slowly mixing (i.e. when $\Phi_{\cC}\paren{t} \sim t^{-\alpha}$with $\alpha < \frac{1}{2}$), the developed \text{ $\cC$-Mix Improved UCB} algorithm has regret upper bound incurring only an additive penalty when compared to the independent outcomes scenario for problem dependent upper bound.
Under certain conditions on the relation between the number of arms, the time horizon and the mixing rate, a proper choice of the threshold in the penalty allows to recover the same regret upper bounds as for independent data observations. In other regimes, our algorithm highlights the surprising effect that the worst-case upper bound does not scale with the number of arms.


\acks{OZ would like to acknowledge the full support of the Deutsche Forschungsgemeinschaft (DFG) SFB 1294 and the mobility support due to the UFA-DFH through the French-German Doktorandenkolleg CDFA 01-18. The work of A. Carpentier is also partially supported by the DFG Emmy Noether grant MuSyAD (CA 1488/1-1), by the DFG - 314838170, GRK 2297 MathCoRe, by the DFG GRK 2433 DAEDALUS.}


\bibliography{mixing_band}
\bibliographystyle{abbrv}
\newpage

\appendix
\section{Proofs of the results}
\label{appendix}

\subsection{Proof of Proposition \ref{prop:cond_mix_prop}}
In the probabilistic setting defined in Section~\ref{sec:setting}, we consider a stochastic process $\paren{X_{t}}_{t \in [T]}$ that is $\Phi_{\cC}-$mixing with rate $\Phi_{C}\paren{t}$, defined on the probability space $ \paren{\mathcal{X}^{T},\cB^{\otimes T},\mbp} $, where $\cX = [0,1]$ and $\cB^{}$ is the Borel $\sigma-$algebra on $\cX$. To avoid the problems with the definiteness of the mixing coefficients $\Phi_{\cC}\paren{t}$ and of correspondent probabilistic structures, without losing of generality we consider in the proof the process $X_{t} := X_{t}\ind{t \leq T} + c\ind{t>T}$ for some constant $c$, which is defined and exists for all $t \in \mbn$ and coincides with $X_{t}$ for $t\leq T$. Now as we consider $\paren{X_{t}}_{t \in \mbn}$, we denote $\Omega = \cX^{\mbn},
\cF=\cB^{\otimes \mbn}$. To justify the claim of the proposition, we need to prove several probabilistic results, which may be of independent interest. For a given epoch $k$ define $\tau_{k}$ to be the corresponding random start of it and denote $\tilde{X_{t}} := X_{\tau_{k} +t}$ the random process whose samples are collected starting from time $\tau_{k}$.
First, consider the simple case where $\tau_{k} = t_{0} \in [T]$ is a constant stopping time. We also use the definitions of $\sigma-$algebra $\cF_{t}$ and measure $\mbp$ as in Section~\ref{sec:setting}.
Consider the regular conditional distribution $\mathbb{P}_{\mathcal{F}_{t_{0}}}[\cdot]
= \mathbb{P}[\cdot|{\mathcal{F}_{t_{0}}}]$ and the corresponding conditional expectation $\ee{\mathcal{F}_{t_0}}{\cdot}$. Remember that $\mathbb{P}_{\mathcal{F}_{t_{0}}}$ is a random
measure depending on $\omega\in \Omega$, and we would like to substitute the fixed measure $\mathbb{P}$
with this random measure 
in Definition~\ref{def:phi_C_mixing}, thus defining, for all $\omega\in \Omega$, the conditional $\Phi_{\mathcal{C}}-$mixing coefficients with respect to a (fixed) time $t_{0}$ as: 
\begin{align}
\label{condcoef}
\Phi_{\mathcal{C}}^{|t_{0}}\paren{t}_{(\omega)} = \sup\set{  \norm{\ee{\mathcal{F}_{t_{0}}}{\varphi\paren{X_{t_{0}+t+s}}|\mathcal{F}_{t_{0}+s}} - \ee{\mathcal{F}_{t_{0}}}{\varphi\paren{X_{t_{0}+t+s}}}}_{L^\infty(\mathbb{P}_{\cF_{t_0}})} ,\varphi \in \mathcal{C}_{1}, s\geq 1 },
\end{align}
where we denote $\ee{\mathcal{F}_{t_{0}}}{\cdot|\mathcal{F}_{t_{0}+s}}$ for a version of the conditional expectation with respect to the iterated conditional distribution $\probb{\mathcal{F}_{t_{0}}}{\cdot}[\cdot|\mathcal{F}_{t_{0}+s}]$. Observe that for fixed $\omega$,
$\ee{\mathcal{F}_{t_{0}}}{\cdot|\mathcal{F}_{t_{0}+s}}$ is itself a random variable depending on $\omega'\in\Omega$, and the $\norm{.}_\infty$ norm control inside~\eqref{condcoef} is meant with respect to $\omega'$ only,
i.e. as an event of probability 1 with respect to $\mathbb{P}_{\mathcal{F}_{t_{0}}}(d\omega',\omega)$ acting on $\omega'$.

At this point, to alleviate technical measurability issues we assume that $\cC_1$ can be replaced by
a countable subset $\cC_1^*$ of test functions without changing the value of~\eqref{condcoef}. This is the case for instance if $\cC_1$ is separable wrt. the supremum norm.
In this case we can exchange the $\sup$ and $\norm{.}_\infty$ operations and
it holds
\begin{align}
\label{condcoef2}
\Phi_{\mathcal{C}}^{|t_{0}}\paren{t}_{(\omega)} = \norm{\sup\set{  \abs{\ee{\mathcal{F}_{t_{0}}}{\varphi\paren{X_{t_{0}+t+s}}|\mathcal{F}_{t_{0}+s}} - \ee{\mathcal{F}_{t_{0}}}{\varphi\paren{X_{t_{0}+t+s}}}}, \varphi \in \mathcal{C}^*_{1}, s\geq 1 }}_{L^\infty(\mathbb{P}_{\cF_{t_0}})},
\end{align}
and the statement $\Phi_{\mathcal{C}}^{|t_{0}}\paren{t}_{(\omega)} \leq C$ takes the form: for fixed $\omega$, the indicator function $\mathbf{1}_A(\omega,\omega')$ of a certain event $A$ in $(\Omega^2,\cF\otimes \cF_{t_0})$, namely of 

$$A:=\bigg \{\sup_{\varphi \in \cC_1^*, s\geq 1} \abs{ \ee{\mathcal{F}_{t_{0}}}{\varphi\paren{X_{t_{0}+t+s}}|\mathcal{F}_{t_{0}+s}} - \ee{\mathcal{F}_{t_{0}}}{\varphi\paren{X_{t_{0}+t+s}}} }_{(\omega,\omega')} > C \bigg \} , $$

has probability~0 with respect to $\mathbb{P}_{\mathcal{F}_{t_{0}}}{(d\omega',\omega)}$, which is a probability
distribution for the variable $\omega'$. What we want in the end is that
this statement holds for $\mathbb{P}$-almost all $\omega$. By Fubini's theorem, it is sufficient to
establish for this that the event $A$ has probability~0 under the joint distribution
$\mathbb{Q} := \mathbb{P} \otimes \mathbb{P}_{\mathcal{F}_{t_{0}}}$ on $(\Omega^2,\cF\otimes \cF_{t_0})$.
However, it can be checked from the very definition of a regular conditional probability distribution
that 
$\ee{(\omega,\omega')\sim \mathbb{Q}}{\mathbf{1}_A(\omega,\omega')} = \ee{\omega\sim \mathbb{P}}{\mathbf{1}_A(\omega,\omega)}$.
Hence, it is sufficient to prove that $\mathbf{1}_A(\omega,\omega) =0$, $\mathbb{P}$-almost surely.

This "diagonal extraction" principle for iterated conditional conditional distributions is
analyzed in detail (in particular concerning measurability issues) by \citet{Kallenberg:17} (see also \citealp{Kallenberg:10}) and we now provide the most striking result of that theory:
\begin{theorem}{(Theorem 6.21 of \citealp{Kallenberg:17})}
	\label{thm:kal_condition}
	For any probability space $\paren{\Omega,\mathfrak{A},\mathbb{P}}$ and  Borel-generated $\sigma-$algebras $\mathcal{F}$,$\mathcal{G} \subset \mathfrak{A}$, we have that it holds
	for $(\mathfrak{A},\mathbb{P})$-almost all $\omega\in \Omega$: 
	\begin{align}
	\mathbb{P}\paren{\cdot|\mathcal{F}}\paren{\cdot|\mathcal{G}}_{(\omega,\omega)} = \mathbb{P}\paren{\cdot|\mathcal{G}}\paren{\cdot|\mathcal{F}}_{(\omega,\omega)} = \mathbb{P}\paren{\cdot|\mathcal{F} \vee \mathcal{G}}_{(\omega)},
	\end{align} 
	where $\mathcal{F} \vee \mathcal{G}$ we define the smallest $\sigma-$ algebra which contains both $\mathcal{F}$ and $\mathcal{G}$.
\end{theorem}
Applying the result of Theorem \ref{thm:kal_condition} to the $\sigma$-algebras $\mathcal{F}_{t_{0}}$ and $\mathcal{F}_{t_{0}+s}$ (note that both of them are Borel generated by the canonical version of $\paren{X_{t}}_{t \in T}$ over $\paren{\mathcal{X}^{t_{0}}, \mathcal{F}^{\otimes t_{0}}}$ and $\paren{\mathcal{X}^{t_{0}+s}, \mathcal{F}^{\otimes t_{0}+s}}$ correspondingly) and noticing that $\mathcal{F}_{t_{0}} \subset \mathcal{F}_{t_{0} + s}$) we deduce that $\mathbb{P}$-a.s.: 
\begin{align*}
\mathbb{P}\paren{\cdot|\mathcal{F}_{t_{0}}}\paren{\cdot|\mathcal{F}_{t_{0}+s}}_{(\omega,\omega)} = \mathbb{P}\paren{\cdot|\mathcal{F}_{t_{0}+s}}_{(\omega)},
\end{align*}
so that $\mathbb{P}$-a.s.:
\[
\ee{\mathcal{F}_{t_{0}}}{\varphi\paren{X_{t_{0}+t+s}}|\mathcal{F}_{t_{0}+s}}_{(\omega,\omega)}
= \ee{}{\varphi\paren{X_{t_{0}+t+s}}|\mathcal{F}_{t_{0}+s}}_{(\omega)}.
\]
Therefore, to obtain the desired statement it is sufficient to control the following {\em non-random} quantity:
\begin{align}
\label{eq:ph_cond_mixing}
\wt{\Phi}_{\mathcal{C}}^{|t_{0}}\paren{t} = \sup\{  \norm{\ee{}{\varphi\paren{X_{t_{0}+t+s}}|\mathcal{F}_{t_{0}+s}} - \ee{}{\varphi\paren{X_{t_{0}+t+s}}|\mathcal{F}_{t_{0}} }}_{L^\infty(\mathbb{P})} | \varphi \in \mathcal{C}^*_{1}, s\geq 1 \}.
\end{align}
With this result the following lemma holds:
\begin{lemma}
	\label{lem:cond_mix_bound}
	For the $\mathcal{C}-$mixing process $\paren{X_{t}}_{t\in T}$ it holds $\mbp$-a.s. that: 
	\begin{align}
	\wt{\Phi}_{\mathcal{C}}^{|t_{0}}\paren{t} \leq 2\Phi_{\mathcal{C}}\paren{t}.
	\end{align}
\end{lemma}
\begin{proof}
	For every fixed $t_{0},s \in \mathbb{N}$ and $\varphi \in \mathcal{C}_{1}$ using the definition the norm and triangle inequality we have
	\begin{align*}
	\norm{\ee{}{\varphi\paren{X_{t_{0}+t+s}}|\mathcal{F}_{t_{0}+s}} - \ee{}{\varphi\paren{X_{t_{0}+t+s}}|\mathcal{F}_{t_{0}} }}_{\infty}& \leq \norm{\ee{}{\varphi\paren{X_{t_{0}+t+s}}|\mathcal{F}_{t_{0}+s}} - \ee{}{\varphi\paren{X_{t_{0}+t+s}} }}_{\infty} \\
	& + \norm{\ee{}{\varphi\paren{X_{t_{0}+t+s}}|\mathcal{F}_{t_{0}}} - \ee{}{\varphi\paren{X_{t_{0}+t+s}} }}_{\infty}
	\end{align*}
	For the first summand on the left hand side we obtain by the definition of mixing coefficients: 
	\begin{align*}
	\norm{\ee{}{\varphi\paren{X_{t_{0}+t+s}}|\mathcal{F}_{t_{0}+s}} - \ee{}{\varphi\paren{X_{t_{0}+t+s}}}}_{\infty} \leq \Phi_{\mathcal{C}}\paren{t}.
	\end{align*}
	Similarly, for the second summand definition of $\mathcal{C}-$weak mixing by monotonicity of weakly $\cC-$mixing coefficients we obtain: 
	\begin{align*}
	\norm{\ee{}{\varphi\paren{X_{t_{0}+t+s}}|\mathcal{F}_{t_{0}}} - \ee{}{\varphi\paren{X_{t_{0}+t+s}} }}_{\infty} \leq \Phi_{\mathcal{C}}\paren{t+s} \leq \Phi_{\mathcal{C}}\paren{t}.
	\end{align*}
\end{proof}
Uniting last two bounds and taking supremum over $\varphi \in \mathcal{C}_{1}$ and $s \geq 1$ in the Equation~\eqref{eq:ph_cond_mixing} we obtain the claim of the Lemma.
\subsubsection{Random stopping time scenario}
Since the starting point $\tau_{k}$ of any epoch $k$ is random itself, we have to extend the argument to the case of random stopping time and show that the mixing property transfers to the process $\tilde{X}_{t}$,
Recall that $\tau_{k}$ is the start of the epoch $k$. 
Define, as usual, the $\sigma-$algebra generated by the random stopping time $\tau_{k}$ : 
\begin{align}
\label{eq:sigma_stop_time}
\cF_{\tau_{k}} := \{ A \in \cB^{\otimes T}: A \cap \{ \tau_{k} \leq t\} \in \cF_{t}, \forall t \in [T] \}
\end{align}
One can readily check that the set in Equation~\eqref{eq:sigma_stop_time} is indeed a $\sigma-$ algebra. Furthermore, we notice that, if $\tau_{k}$ a stopping time, and $s \in \mbn$ some fixed number, then we have that $\tau_{k}+s$ is a stopping time and the corresponding $\sigma$ algebra is defined as:
\begin{align}
\label{eq:sigma_stop_time_const}
\cF_{\tau_{k}+s} := \{ A \in \cB^{\otimes T}: A \cap \{ \tau_{k} +s \leq t\} \in \cF_{t}, \forall t \in [T] \}.
\end{align}
Furthermore, for all $s \in \mbn \setminus \{0\}$ we have $\cF_{\tau_{k}} \subset \cF_{\tau_{k}+s}$. We denote by $\mbp_{\cF_{\tau_{k}}} := \mbp\paren{\cdot|\cF_{\tau_{k}}}$ the regular conditional distribution of $\mbp$ conditioned by the $\sigma-$algebra $\cF_{\tau_{k}}$ generated by the stopping time $\tau_{k}$. Analogously to Definition~\ref{condcoef} in the fixed time scenario, for all $\omega \in \Omega$ we can define the conditional mixing coefficients with respect to the random time $\tau_{k}$: 
\begin{align}
\label{eq:condcoef_randtime}
\Phi_{\mathcal{C}}^{|\tau_{k}}\paren{t}_{(\omega)} = \sup\set{  \norm{\ee{\mathcal{F}_{\tau_{k}}}{\varphi\paren{X_{\tau_{k}+t+s}}|\mathcal{F}_{\tau_{k}+s}} - \ee{\mathcal{F}_{\tau_{k}}}{\varphi\paren{X_{\tau_{k}+t+s}}}}_{L^{\infty}\paren{\mbp_{\cF_{\tau_{k}}}}},\varphi \in \mathcal{C}^{\star}_{1}, s\geq 1 }.
\end{align}
Repeating the argumentation of the fixed time case  with $\sigma-$algebras $\cF_{\tau_{k}}$ and $\cF_{\tau_{k} + s}$ correspondingly  one obtains that in order to control the \textit{random} quantity $\Phi_{\mathcal{C}}^{|\tau_{k}}\paren{t}_{(\omega)} $ for $\mbp-$almost all $\omega$ it is sufficient to control: 
\begin{align}
\label{eq:ph_cond_mixing_rand_time}
\wt{\Phi}_{\mathcal{C}}^{|\tau_{k}}\paren{t} = \sup\{  \norm{\ee{}{\varphi\paren{X_{\tau_{k}+t+s}}|\mathcal{F}_{\tau_{k}+s}} - \ee{}{\varphi\paren{X_{\tau_{k}+t+s}}|\mathcal{F}_{\tau_{k}} }}_{L^\infty(\mathbb{P})} | \varphi \in \mathcal{C}^*_{1}, s\geq 1 \}.
\end{align}
Recalling that $\tilde{X}_{t}^{\tau_{k}}:= X_{\tau_{k} +t}$ we can show that the following result is true: 
\begin{theorem}
	\label{thm:rand_time_mix_cont}
	Let $\tau_{k}$ be the stopping time from Equation~\eqref{eq:sigma_stop_time}, such that $\prob{\tau_{k} < \infty} = 1$. For all $t,s > 0$ and any $\varphi \in \mathcal{C}_{1}^{\star}$ it holds: 
	\begin{align*}
	\norm{\ee{\cF_{\tau_{k}}}{\varphi\paren{\tilde{X}_{t+s}}} - \ee{}{\varphi\paren{\tilde{X}_{t + s}}}}_{\infty} \leq \Phi_{\mathcal{C}}\paren{t}.
	\end{align*}
\end{theorem}

\begin{proof}
	Denote $Y_{\tau_{k}+t+s} := \varphi\paren{\tilde{X}_{t+s}} - \ee{}{\varphi\paren{\tilde{X}_{s+t}}}$	. Obviously $Y_{\tau_{k}+s+t}$ is measurable w.r.t. the $\sigma-$algebra $\cF_{\tau_{k}+s+t}$ and centered random variable. Since $\mbi_{\tau_{k}=\ell}$ is $\cF_{\tau_{k}}-$measurable for any $\ell \in [T]$ and $\cF_{\tau_{k}} \subset \cF_{\tau_{k}+s}$ we can write: 
	\begin{align*}
	\ee{}{Y_{\tau_{k}+s+t}|\cF_{\tau_{k}+s}} = \ee{}{\sum_{\ell \in T} \mbi_{\tau_{k} = \ell}Y_{\tau_{k}+s+t}|\cF_{\tau_{k}+s}}= \sum_{\ell \in T}\mbi_{\tau_{k}=\ell}\ee{}{Y_{\ell + s+ t}|\cF_{\tau_{k}+s}}.
	\end{align*}
	We make use of the following general measure-theoretic Lemma. 
	\begin{lemma}
		\label{lem: sigma_alg_eq}
		Let $Z$ be an integrable real-valued random variable, defined on the space $\paren{\mathcal{X}^{T},\cB^{T},\mbp}$. Then $\mbp-$a.s., it holds that for any $\ell \in T$: 
		\begin{align}
		\mbi_{\tau_{k}=\ell}\ee{}{Z|\cF_{\tau_{k}+s}} = \mbi_{\tau_{k} = \ell}\ee{}{Z|\cF_{\ell+s}}
		\end{align}
	\end{lemma}
	\begin{proof}
		For any $\ell \in T$, from the definition of the $\sigma-$algebra $\mathcal{F}_{\tau_{k}+s}$, we have $\{ \tau_{k}\leq \ell \} \in \cF_{\ell} \subset \cF_{\ell+s}$, as well as $\{ \tau_{k}\geq \ell \} \in \cF_{\ell} \subset \cF_{\ell+s}$. It follows, that for any $A \in \cF_{\tau_{k}+s}$, we have $A \cap \{ \tau_{k}=\ell \} = A \cap \{ \tau_{k} \leq \ell  \} \cap \{\tau_{k} \geq \ell \} \in \cF_{\ell + s}$. Thus, for all $A \in \cF_{\tau_{k}+s}$ on the one hand, by conditioning on $\cF_{\ell+s}$ we get: 
		\begin{align*}
		\ee{}{\mbi_{A}\mbi_{\tau_{k} = \ell} Z} = \ee{}{\mbi_{A}\mbi_{\tau_{k} = \ell} \ee{\cF_{\ell+s}}{Z}}.
		\end{align*}
		But, on the other hand, by conditioning on $\cF_{\tau_{k}+s}$ we have: 
		\begin{align*}
		\ee{}{\mbi_{A}\mbi_{\tau_{k} = \ell} Z} = \ee{}{\mbi_{A}\mbi_{\tau_{k} = \ell} \ee{\cF_{\tau_{k}+s}}{Z}}.
		\end{align*}
		Denote $W:= \mbi_{\tau_{k} = \ell}\ee{\cF_{\ell+s}}{Z}$. If we ensure that $W$ is $\cF_{\tau_{k} + s}$ measurable then $\mbi_{\tau_{k} = \ell}\ee{\cF_{\ell+s}}{Z}$ is a version of $\mbi_{\tau_{k} = \ell}\ee{\cF_{\tau_{k}+s}}{Z}$ and thus the claim is proved. 
		For any Borel set $B \in \cB^{\otimes T}$, such that $0$  is not in $B$ we have: 
		\begin{align*}
		\{W \in B \} = \{\tau_{k}=\ell \} \cap \{\ee{\cF_{\ell + s}}{Z} \in B \} \in \cF_{\tau_{k} + s},
		\end{align*} 
		since $\{ \tau_{k}=\ell  \} \in \cF_{\tau_{k} + s}$ and $\{ \ee{\cF_{\ell + s}}{Z} \in B \} \in \cF_{\ell + s}$ (which follows from the definition of conditional expectation). Otherwise, we have that:
		\begin{align*}
		\{ W =0 \} = \{\tau_{k} \neq \ell \} \cup \{ \{ \tau_{k} = \ell \} \cap \{\ee{\cF_{\ell+s}}{Z} = 0\} \} \in \cF_{\tau_{k} + s},
		\end{align*} 
		since $\{ \tau_{k} \neq \ell \} \in \cF_{\tau_{k} + s} $ and $\{\tau_{k} = \ell \} \cap \{ \ee{\cF_{\ell + s}}{Z}=0 \} \in \cF_{\ell + s}$. Uniting both cases gives that $W$ is $\cF_{\tau_{k} + s}-$measurable and thus $\mbp-$a.s. we have: 
		\begin{align*}
		\mbi_{\tau_{k} = \ell}\ee{\cF_{\tau_{k} + s}}{Z} = \mbi_{\tau_{k} = \ell}\ee{\cF_{\ell} + s}{Z}.
		\end{align*} 
	\end{proof}
	Applying Lemma \ref{lem: sigma_alg_eq} for $Z:=Y_{\ell + s+ t}$ (for any admissible $\ell,s,t$ ), which is real-valued and defined on the space $\paren{\mathcal{X}^{T},\cB^{\otimes T},\mbp}$ we deduce that it holds $\mbp-$a.s : 
	\begin{align*}
	\ee{}{Y_{\tau_{k}+s+t}|\cF_{\tau_{k}+s}}  &= \sum_{\ell \in T}\mbi_{\tau_{k}=\ell}\ee{}{Y_{\ell + s+ t}|\cF_{\tau_{k}+s}} = \sum_{\ell \in T}\mbi_{\tau_{k}=\ell}\ee{}{Y_{\ell + s+ t}|\cF_{\ell+s}} \\
	& \leq \sum_{\ell \in T} \mbi_{\tau_{k} = \ell} \Phi_{\cC}\paren{t} = \Phi_{\cC}\paren{t}.
	\end{align*}
	Therefore, we finally get: 
	\begin{align*}
	\norm{\ee{\cF_{\tau_{k}}}{\varphi\paren{\tilde{X}_{t+s}}} - \ee{}{\varphi\paren{\tilde{X}_{t + s}}}}_{L_{\infty}\paren{\mbp}} = \norm{	\ee{}{Y_{\tau_{k}+s+t}|\cF_{\tau_{k}+s}} }_{L_{\infty}\paren{\mbp}} \leq \Phi_\cC\paren{t},
	\end{align*}
	and the claim is proved.
\end{proof}
Using the same ideas as in the proof of Lemma\ref{lem:cond_mix_bound} we obtain the proof the claim. 

\subsection{ \textbf{Proof of  Proposition~\ref{prop:cond_mix_prop} 
}}
\begin{proof}
	Using representation of the conditional mixing coefficients \ref{eq:condcoef_randtime}, the triangle inequality and the result of Proposition \ref{thm:rand_time_mix_cont} we obtain for any $\varphi \in \cC_{1}^{\star}$: 
	
	\begin{multline*}
	\norm{\ee{\cF_{\tau_{k}}}{\varphi\paren{\tilde{X}_{t+s}}|\cF_{\tau_{k}+s}} - \ee{\cF_{\tau_{k}}}{\varphi\paren{\tilde{X}_{t+s}}}}_{L_{\infty}\paren{\mbp}}\\
	\begin{aligned}
	&= \norm{\ee{}{\varphi\paren{\tilde{X}_{t+s}}|\cF_{\tau_{k}+s}} - \ee{}{\varphi\paren{\tilde{X}_{t+s}}|\cF_{\tau_{k}}}}_{L_{\infty}\paren{\mbp}} \\
	& \leq \norm{\ee{}{\varphi\paren{\tilde{X}_{t+s}}|\cF_{\tau_{k}+s}} - \ee{}{\varphi\paren{\tilde{X}_{t+s}}}}_{L_{\infty}\paren{\mbp}}\\
	& \qquad+ \norm{\ee{}{\varphi\paren{\tilde{X}_{t+s}}|\cF_{\tau_{k}}} - \ee{}{\varphi\paren{\tilde{X}_{t+s}}}}_{L_{\infty}\paren{\mbp}} \\
	& \leq 2 \Phi_{\cC}\paren{t}.
	\end{aligned}
	\end{multline*}
	Taking the supremum over $\varphi \in \cC_{1}^{\star}$ and all $s\geq 1$, we obtain the claim of the proposition. 
\end{proof}

\subsection{Proof of Theorem \ref{thm: upp_bound}}

Firstly, we prove the claim of an intermediate result which ensures the optimal choice of number of pulls in the slow mixing scenario for the given epoch $s$ in the learning scheme of Algorithm \ref{alg:cmix_ucb}. 

\begin{lemma}
	\label{lem:sample_choice}
	Consider the fixed epoch $s \in \mathbb{N}$ and let  $\theta_{s},\delta_{s}$ and $\Omega\paren{\theta_{s},b_{s}}$ be chosen as in Algorithm~\ref{alg:cmix_ucb}.
	Define $B_{s}$ to be the set of active arms at the epoch $s$ and $\varsigma_{s}$ the corresponding pulling strategy. For all arms $j \in A_{0}$ we assume that their mixing rates  are bounded by the rate $\Phi_{\cC}(t) = t^{-\alpha}$, where $\alpha \in (0,1/2)$. 
	Furthermore, denote 
$c_{0} = \paren{\paren{1-\alpha}\paren{1/2-\alpha}}^{-1}$, $c_{1} = \paren{ \frac{\paren{1-\alpha}{1/2-\alpha}}{80}}^{\frac{2}{1-2\alpha}}$, $c_{3} = 12800c_{0}$ and $s_{end}$ for the last possible epoch. If the number of pulls $T_{s}$ of each arm $j \in B_{s}$ at the epoch $s$ is chosen to be: 
	\begin{align*}
	T_{s} = T_{s,1}&:= \bigg\lceil \frac{32 \log\paren{\mathcal{A}T\theta_{s}^{2}}}{\theta_{s}^{2 }}\bigg\rceil, \text{ for } b_{s} \geq  \paren{32c_{1}^{-1}\theta_{s}^{-2}\log\paren{\mathcal{A}T\theta_{s}^{2}}}^{\frac{1-2\alpha}{2\alpha}};\\
	T_{s} = T_{s,2}& := \Bigg\lceil \frac{1}{b_{s}} \paren{\frac{c_3 \log\paren{\mathcal{A}T\theta_{s}^{2}}}{\theta_{s}^{2 }}}^{\frac{1}{2\alpha}}\Bigg\rceil, \text{ for } b_{s} \leq \paren{32c_1^{-1}\theta_{s}^{-2}\log\paren{\mathcal{A}T\theta_{s}^{2}}}^{\frac{1-2\alpha}{2\alpha}};
	\end{align*}
	then it holds that $\Omega_{}(\theta_{s},b_{s}) \leq \frac{\theta_{s}}{2}$ for $s \in \{0,\ldots,s_{end}\}$.
\end{lemma}
\begin{proof}
	Without loss of generality we enumerate all arms in $B_{s}$ as $\{1,\ldots,b_{s}\}$. For an arm $j \in B_{s}$ we pull it according to the equispaced
	schedule $\varsigma_{s}^{j}$ defined as follows:
	\begin{align*}
	\varsigma_{s}^{j} = \paren{j, j+b_{s}, \ldots, j +(T_{s}-1)b_{s}}.
	\end{align*}
	Clearly, the total number of pulls of arm $j$ during the epoch is $\abs{\varsigma_{s}^{j}}=T_{s}$ for each $j \in \{1,2,\ldots,b_{s}\}$. We recall that $b_{s}$ denotes the number of active arms at epoch $s$. To estimate the contamination in the Hoeffding's inequality  due to the dependence we  upper bound the contribution of dependence sum in Equation\eqref{eq:hoeff_gap} with number of samples $T_{s}$ and time gap $b_{s}$. 	We analyse the case $0\leq \alpha \leq \frac{1}{2}$. 

	Namely, assuming that $0<\alpha < \frac{1}{2}$ we have:
	\begin{align*}
			\sum_{j=1}^{T_{s}}{j}^{-\frac{3}{2}} \sum_{\ell=1}^{j}\Phi_{\cC}(b_{s}\ell) = b_{s}^{-\alpha}\sum_{j=1}^{T_{s}}j^{-\frac{3}{2}}\sum_{\ell=1}^{j}\ell^{-\alpha} \leq c_{0} b_{s}^{-\alpha}T_{s}^{\frac{1}{2} -\alpha},
	\end{align*}
	where we twice used approximation of the sum by integral and $c_{0} = \paren{\paren{1-\alpha}\paren{1/2-\alpha}}^{-1}$. 

Notice that with this upper bound for the confidence term $\Omega\paren{\theta_s, b_{s}}$ from Algorithm~\ref{alg:cmix_ucb} in the epoch $s$ we get: 
\begin{align*}
\Omega\paren{\theta_s,b_s} \leq 2 \max\paren{1,80c_{0}b_s^{-\alpha}T_{s}^{1/2-\alpha}}\sqrt{\frac{2\log\paren{\cA T \theta_{s}^2}}{T_s}}.
\end{align*}
	Therefore, if for $b_{s} > \paren{32c_1^{-1}\theta_{s}^{-2}\log\paren{\mathcal{A}T\theta_{s}^{2}}}^{\frac{1-2\alpha}{2\alpha}}$ provides that $80c_{0}b_s^{-\alpha}T_{s}^{1/2 - \alpha}\leq 1$ so by choosing 
	\begin{align*}
	T_{s} := \bigg\lceil \frac{32 \log\paren{\mathcal{A}T\theta_{s}^{2}}}{\theta_{s}^{2 }}\bigg\rceil,
	\end{align*}
we assure that $\Omega\paren{\theta_{s},b_{s}} \leq \frac{\theta}{2}$. 

Similarly, in case when  $b_{s} \leq \paren{32c_1^{-1}\theta_{s}^{-2}\log\paren{\mathcal{A}T\theta_{s}^{2}}}^{\frac{1-2\alpha}{2\alpha}}$ 
	we have that $80c_{0}b_s^{-\alpha}T_{s}^{1/2 - \alpha} > 1$ so by choosing 
	\begin{align*}
	T_{s} :=\bigg\lceil \frac{1}{b_{s}} \paren{\frac{12800 \log\paren{\mathcal{A}T\theta_{s}^{2}}}{\theta_{s}^{2}}}^{\frac{1}{2\alpha}}\bigg\rceil,
	\end{align*}
	we assure also that $\Omega\paren{\theta_{s},b_{s}} \leq \frac{\theta}{2}$.
	Therefore the lemma is proved.
\end{proof}
Combining the choice of the number of pulls $T_{s}$ given in Lemma~\ref{lem:sample_choice}, the theoretical argument that the $\Phi_{\cC}-$mixing property of the processes $\paren{X_{\tau_{s}+t}}_{t \in [T]}$ is preserved under (provided that $\tau_{s}$ is a stopping time) and concentration bounds \eqref{eq:hoeff_gap} we now prove the main result. 
\begin{proof}
	
	Recall that we denote for $\lambda\geq 0$ the set $A_{\lambda}$  as the set of suboptimal arms $i$ for which $\{ \Delta_{i}> \lambda \}$ so that $A_{0}$ is the overall set of suboptimal arms. 
	Recall that $\Delta_{k}:= \mu_{\star}-\mu_{k}$ for $k \in A_{0}$ and we define $\Delta_{\star,\lambda} = \min_{j \in A_{\lambda}} \Delta_{j}$. For an epoch $s$ consider the confidence bound $\Omega\paren{\theta_{s},b_s}$ as in Algorithm~\ref{alg:cmix_ucb}.
	where $b_{s}$ is the number of active arms during the epoch $s$ (which is a random quantity, but conditional on the start $\tau_{s}$ of the epoch $s$ is deterministic), $\theta_{s}= 2^{-s}$.
	Since $R(T)=\ee{}{\sum_{k \in A_{0}}\Delta_{k}N_{k}\paren{T}} = \sum_{k \in A_{0}} \Delta_{{k}}\ee{}{N_{k}\paren{T}}$, the main target is to upper bound the number of pulls of each arm $k \in A_{0}$. We suppress index $T$ in $N_{k}\paren{T}$ for simplicity.
	For every suboptimal arm $i$ define $m_{i} := \min\{ m \in \mathbb{N}: \theta_{m} \leq \frac{\Delta_{i}}{2}\}$. From the definition of $m_i$ it follows that $\theta_{{m}_{i}} < \frac{\Delta_{i}}{2} \leq \theta_{m_{i}-1}$.  Solving this as the inequality in $\frac{1}{\Delta_{i}}$ we get: 
	\begin{align}
	\label{eq:sup_bounds}
	\frac{1}{\theta_{m_{i}}} \leq \frac{4}{\Delta_{i}} < \frac{1}{\theta_{m_{i}+1}}.
	\end{align}
	Resolving in $\theta_{m_{i}}$ we obtain:
	\begin{align*}
	\frac{\Delta_{i}}{4} \leq \theta_{{m}_{i}} < \frac{\Delta_{i}}{2}.
	\end{align*} 		
	
	 We fix some optimal arm (to which we later refer at to $\star$) and denote by $M_\star$ the first epoch when this optimal arm $\star$
	has been eliminated. Note that it is possible that $M_{\star} = \infty$ and that it is enough to consider only certain optimal arm for the further analysis.  Also let $m_{\lambda} := \min \{m| \theta_{m} < \frac{\lambda}{2} \} $, which implies that forall $i \in A_{\lambda}$ we have $\Delta_{i} \leq \lambda$.

	For the arm $k \in A_{\lambda}$ let $M_{k}$ denote the (random) epoch at which arm is $k$ is eliminated and introduce the following event:
	\begin{align*}
			\cE_{k} = \{\tau_{m_{k}+1} \leq T , M_{k} \leq m_{k}\} \cup \{ \tau_{m_{k}+1} >T \}. 
	\end{align*} 
	Notice, that $\cE_{k}^{c} = \{ \tau_{m_{k}+1} \leq T, M_{k}> m_{k}\} $ which means that the arm $k$ is eliminated after epoch $m_{k}$ and $m_{k}$ is finished.
	Using the definition of event $\cE_{k}$ and introducing a threshold $\lambda$, we can decompose the pseudo-regret into the following parts: 
	\begin{align*}
	R(T)& =\ee{}{\sum_{k \in A_{0}}\Delta_{k}N_{k}} = \sum_{k \in A_0} \e{N_k}\Delta_k\\
	& \leq \sum_{k \in A_\lambda} \e{N_k} \Delta_k
	+ \lambda \e{\sum_{k \in A_0 \setminus A_\lambda} N_k} \\
	&  = \sum_{k \in A_\lambda} \e{N_k\ind{\cE_k}} \Delta_k
	+ \sum_{k \in A_\lambda} \e{N_k\ind{\cE_k^c}} \Delta_k + \lambda \e{\sum_{k \in A_0 \setminus A_\lambda} N_k}. 
	\end{align*}
	We analyze the contributions from each of the sums in the last inequality separately. Clearly, the last sum can be bounded by $\lambda T$.
	
	We have that this sum can be decomposed in the following way: 
	\begin{align}
	\label{eq:sec_sum}
	\sum_{k \in A_\lambda} \Delta_k \e{N_k\ind{\cE_k^c}}
	=  \sum_{k \in A_\lambda} \Delta_k \e{N_k\ind{\cE_k^c} \ind{M_{\star} < m_k} }
	+   \sum_{k \in A_\lambda} \Delta_k \e{N_k\ind{\cE_k^c} \ind{M_{\star} \geq m_k} }.
	\end{align}
	Consider the second sum on the right hand side in Equation~\eqref{eq:sec_sum}. 
	For a fixed arm $k \in A_{\lambda}$, the confidence level of the epoch $s$ is selected to be $\delta_{s} = \frac{1}{T\theta^{2}_{s}}$.  For each $k \in [K]$, $s \in \{0,\ldots, s_{\text{end}}\}$ we consider events $D_{k,s}$ and $E_{\star,s}$ whose complements are given as: 
	
	\begin{align*}
		D^{c}_{k,s} &:= \{k \in B_{s}, \tau_{s+1} \leq T,\hat{\mu}_{k,s}^{} \leq \mu_{k} +
	\Omega\paren{\theta_{s},b_{s}}
	\}, 
	\\
	E^{c}_{\star,s} &:= \{\star \in B_s, \tau_{s+1} \leq T,\hat{\mu}_{\star,s}^{} \geq \mu_{\star} -
	\Omega\paren{\theta_{s},b_{s}} \},
	\end{align*}
%
%
	where $\Omega\paren{\theta_{s},b_{s}}$ is as in Algorithm~\ref{alg:cmix_ucb}. We remark, that conditions $\{ \tau_{s+1} \leq T \}$ and either $\{ k \in B_s \}$ or $\{ \star \in B_s \}$ are added to asure that the computation procedure of Algorithm~\ref{alg:cmix_ucb} can be formally completed in the epoch $s$.
	By Proposition \ref{prop:cond_mix_prop}, the samples which are collected during epoch $s$ from the process $X_{\tau_{s}+t}^{k}$, $k \in B_{s}$ satisfy, conditionally
	to the past $\cF_{\tau_s}$, a weak-mixing assumption with rate $2\Phi_{C}\paren{t}$. Therefore, we can apply (conditioned on the information given at the beginning of the epoch) the concentration inequality \eqref{eq:hoeff_gap} for the weak-mixing process $X_{\tau_{s}+t}^{k}$ to control the probabilities of the "bad" events $D_{k,s}^{c}, E_{\star,s}^{c}$ we have that $\mbp-$almost surely:
	\begin{align}
	\label{eq: bad_events_prob}
	\probb{\cF_{\tau_s}}{D_{k,s}^{c}} \leq \delta_{s}, \qquad \probb{\cF_{\tau_s}}{E^{c}_{\star,s}} \leq \delta_{s}.
	\end{align}
	For the moment we are interested in the case where $s = m_{k}$. Notice that the event $\ind{\cE_k^c} \ind{M_{\star} \geq m_k}$ implies that arm $k$ has not been eliminated until the epoch $m_{k}$, while arm $\star $ belongs to the set of active arms $B_{m_{k}}$ and that the epoch $m_{k}$ has been completed. Furthermore, one can readily check that event $\cE_{k}^{c} \cap \{M_{\star} \geq m_{k}\}$ implies that the event $D_{k,m_{k}}^{c} \cup E_{\star,s}^{c}$ holds. To prove this notice that on the event  $\cE_{k}^{c} \cap \{M_{\star} \geq m_{k}\}$ it obviously holds that $\tau_{m_{k}+1} \leq T$, $\star \in B_{m_{k}}$ and $k \in B_{m_{k}}$. Now, if neither $\{ \hat{\mu}_{k,m_{k}} > \mu_{k} + \Omega\paren{\theta_s,b_s} \}$ nor $\{ \hat{\mu}_{\star,m_{k}} > \mu_{\star} + \Omega\paren{\theta_s,b_s} \}$ holds then arm $k$ will be eliminated at the end of epoch $m_{k}$. Indeed, by using Lemma~\ref{lem:sample_choice} and Inequality~\eqref{eq:sup_bounds},  we infer that $\Omega\paren{\theta_{m_{k}},b_{m_{k}}} \leq \frac{\theta^{}_{m_{k}}}{2} \leq  \frac{\Delta_{k}}{4} $. Consequently, this implies the following chain of inequalities:
	\begin{align*}
	\hat{\mu}_{k,s} + \Omega\paren{\theta_{m_{k}},k_{m_{k}}}& \leq \mu_{k} + 2\Omega\paren{\theta_{m_{k}},b_{m_{k}}} \\  
	& \leq \mu_{k} + \Delta_{k} - 2 \Omega\paren{\theta_{m_{k}},b_{m_{k}}}\\ 
	& = \mu_{\star}-2\Omega\paren{\theta_{m_{k}},b_{m_{k}}}\leq \hat{\mu}_{\star,s}^{} - \Omega\paren{\theta_{m_{k}},b_{m_{k}}},
	\end{align*}
	and the arm $k$ is eliminated due to the scheme of Algorithm~\ref{alg:cmix_ucb}. Therefore, we have $\cE_{k}^{c} \cap \{M_{\star} \geq m_{k}\} \subset D_{k,m_{k}}^{c} \cup E_{\star,m_{k}}^{c}$. 
	Furthermore, notice that by conditioning on the $\sigma-$algebra $\cF_{\tau_{m_k}}$ of the events preceeding the epoch $m_{k}$ we get: 
	\begin{align*}
	\ee{}{\ind{\cE_k^c} \ind{M_{\star} \geq m_k}} = \ee{}{ \ee{\cF_{\tau_{m_k}}}{\ind{\cE_k^c} \ind{M_{\star} \geq m_k}}} \leq  \ee{}{ \probb{\cF_{\tau_{m_{k}}}}{D_{k,m_k}^{c} \cup E_{\star,m_k}^{c}}}\leq \frac{2}{\mathcal{A}T\theta^{2}_{m_{k}}},
	\end{align*}
	where in the last inequality we used the union bound over events $E_{\star,m_{k}}^{c}, D_{k,m_{k}}^{c}$ and their control by means of concentration during the epoch $m_{k}$.

	Thus, bounding the number of pulls $N_k$ trivially by $T$ and plugging in the bound on the $\ee{}{\ind{\cE_k^c} \ind{M_{\star} \geq m_k}}$, we obtain for this part of the pseudo-regret the following upper bound: 
	\begin{align*}
	\sum_{k \in A_\lambda} \Delta_k \e{N_k\ind{\cE_k^c} \ind{M_{\star} \geq m_k} } \leq \sum_{ k \in A_{\lambda}} \Delta_{k} T \frac{2}{\mathcal{A}T\theta^{2}_{m_{k}}}  \leq \frac{32}{\mathcal{A}} \sum_{ k \in A_{\lambda}}\frac{1}{\Delta_{{k}}},
	\end{align*}
	where in the last inequality we used the relation \eqref{eq:sup_bounds} between $\theta_{m_{k}}$ and $\Delta_{m_{k}}$.
	
	We focus now on the first sum term in the right hand side of Equation~\eqref{eq:sec_sum}. By changing the order of summation over the epochs and counting the regret from the active arms in each epoch $s$, we obtain:
	\begin{align*}
	\sum_{k \in A_\lambda} \Delta_k \e{N_k\ind{\cE_{k}^c} \ind{M_{\star} < m_k} }
	& \leq \sum_{k \in A_\lambda} \Delta_k
	\sum_{s < m_k} \e{N_k\ind{M_{\star} = s} \ind{\tau_{m_{k}+1} \leq T, M_{k} >m_{k}} }\\
	& = \sum_{s=0}^{m_\lambda}
	\sum_{k: m_k > s} \Delta_k \e{N_k\ind{M_{\star} = s}\ind{\tau_{m_{k}+1} \leq T,M_{k}>m_{k}} } \\
	& \leq \sum_{s=0}^{m_\lambda} 2\theta_s 
	\e[4]{ \ind{M_{\star} = s}\ind{\tau_{s+1} \leq T} \sum_{k: m_k > s}  N_k } \\
	& \leq 2\sum_{s=0}^{m_\lambda} T {\theta}_{s} \prob{M_{\star}=s, {\tau_{s+1} \leq T}}.
	\end{align*}
	Recall that through $B_{s}$ we denote the set of active arms at the epoch $s$ and $b_s$ is its cardinality. Since the event ${M_{\star} = s}$ means that the optimal arm was eliminated by some active arm $k$ in the epoch $s$ we have: 
	\begin{align*}
	\prob{M_{\star}=s,\tau_{s+1}\leq T} &
	\leq \e{ \ind{* \in B_s, \tau_{s+1} \leq T} \sum_{k \in B_s} \ind{\wh{\mu}_{k,s}^{} > \wh{\mu}_{\star,s}^{} + 2 \Omega(\theta_{s},b_{s})}}\\
	&  \leq \sum_{k: m_k \geq s} \e{ \ind{k \in B_s ; * \in B_s; \tau_{s+1} \leq T, \wh{\mu}_{k,s}^{} > \wh{\mu}_{\star,s} + 2 \Omega(\theta_{s},b_{s})} }\\
	& +  \e{\sum_{k: m_k < s} \ind{* \in B_s ; k \in B_s, \tau_{s+1} \leq T}}\\
	&  \leq \sum_{k: m_k \geq s} \ee{}{ \ee{\cF_{\tau_{s}}}{\ind{k \in B_s ; * \in B_s; \tau_{s+1} \leq T ,\wh{\mu}_{k,s}^{} > \wh{\mu}_{\star,s} + 2 \Omega(\theta_{s},b_{s})}} } 
	\\
	& +  \sum_{k: m_k < s}\ee{}{ \ee{\cF_{\tau_{m_k}}}{ \ind{* \in B_s ; k \in B_s, \tau_{s+1} \leq T}} }\\
	& \leq \sum_{k: m_k \geq s} \frac{2}{\mathcal{A}T \theta_s^2}
	+ \sum_{k: m_k < s} \ee{}{\probb{\cF_{\tau_{m_k}}}{\cE_k^c \cap \set{M_{\star} \geq m_k}}}\\
	& \leq \sum_{k: m_k \geq s} \frac{2}{ \mathcal{A}T \theta_s^2}
	+ \sum_{k: m_k < s} \frac{2}{ \mathcal{A}T \theta_{m_k}^2},  \\
	\end{align*}
	where 
	we used tower property for expectations and that for conditional probabilities it holds almost surely $\probb{\cF_{\tau_{s}}}{k \in B_s ; * \in B_s; \wh{\mu}_{k,s}^{} > \wh{\mu}_{\star,s} + 2 \Omega(\delta_{s},\varsigma_{s})} \leq \probb{\cF_{\tau_{s}}}{D_{k,s}^{c} \cup E_{\star,s}^{c}}$ as well as  $\probb{\cF_{\tau_{m_k}}}{\cE_k^c \cap \set{M_{\star} \geq m_k}} \leq \probb{\cF_{\tau_{m_{k}}}}{\cE_k^c} \leq \probb{\cF_{\tau_{m_k}}}{D_{k,m_k}^{c} \cup E_{\star,m_k}^{c}} $ and the control of the event's probabilities in the epoch $s$ given by Equation~\eqref{eq: bad_events_prob}.
	Plugging this bound into the previous result and using the definition of the sequence $\theta_{s}$ we obtain the following upper bound:
	\begin{align*}
	2\sum_{s=0}^{m_\lambda} T {\theta}_{s} \prob{M_{\star}=s, \tau_{s+1}\leq T} &\leq \frac{4}{\mathcal{A}}\sum_{s=0}^{m_\lambda} \theta_{s} \paren{\sum_{k: m_k \geq s}
		\frac{1}{\theta^2_s} + \sum_{k: m_k < s} \frac{1}{\theta^2_{m_k}}}  \\
	& \leq \frac{4}{\mathcal{A}}\sum_{k \in A_0}
	\paren{ \sum_{s \leq m_k\wedge m_\lambda} \frac{1}{\theta_s}
		+ \frac{1}{\theta^2_{m_k}}\sum_{ s=m_k+1}^{m_\lambda} \theta_s}\\
	& \leq \frac{8}{\mathcal{A}} \sum_{k \in A_0}
	\paren{ \frac{1}{\theta_{m_k\wedge m_\lambda}} + \frac{\ind{m_k \leq m_\lambda}}{\theta_{m_k}}} \\
	&  \leq \frac{64}{\mathcal{A}} \paren{\sum_{k \in A_\lambda} \frac{1}{\Delta_k}
		+ \sum_{k \in A_0\setminus A_\lambda} \frac{1}{\lambda} },
	\end{align*}
	
	Gathering upper bounds for each sum in Equation~\eqref{eq:sec_sum} we obtain: 
	\begin{align}
	\label{eq:late_elim}
	\sum_{k \in A_\lambda} \Delta_k \e{N_k\ind{\cE_k^c}} \leq \frac{96}{\mathcal{A}} \sum_{ k \in A_{\lambda}}\frac{1}{\Delta_{{k}}}  + \frac{64}{\mathcal{A}} \sum_{ k \in A_{0}\setminus A_{\lambda}} \frac{1}{\lambda}. 
	\end{align}
	
	Finally, for the contribution of $ \ee{}{\sum_{k \in A_\lambda} {N_k\ind{\cE_k}} \Delta_k}$ we provide the almost surely analysis for the quantity under expectation. Firstly, notice that on the event $\ind{\cE_{k}}$, each arm is pulled until it will be eliminated at latest at the epoch $m_{k}$. Thus, recalling that for any $i \in A_{\lambda}$ $\Delta_{i} \leq \lambda$ and using simply $T_{s} \leq T_{s,1}+T_{s,2}$ (where $T_{s,1}, T_{s,2}$ are given by Lemma~\ref{lem:sample_choice}) we can write the following chain of inequalities which hold almost surely: 
	\begin{align*}
	\sum_{k \in A_\lambda} {N_k\ind{\cE_k}} \Delta_k &=  \sum_{k \in A_\lambda} \sum_{ s = 0}^{s_{end} \wedge m_\lambda } \Delta_{{k}}{T_{s}\ind{\cE_k} \ind{ k \in B_{s}}} \leq \sum_{k \in A_\lambda} \sum_{ s = 0}^{s_{end} \wedge m_\lambda } \Delta_{{k}}{ \paren{T_{s,1} + T_{s,2}}\ind{\cE_k} \ind{ k \in B_{s}}}   \\
	& \leq \sum_{k \in A_{\lambda}}\Delta_{{k}}\sum_{s=0}^{s_{\text{ end}\wedge m_{\lambda}}}T_{s,1}\ind{\cE_{k}}\ind{k \in B_s} + \sum_{s=0}^{s_{end}\wedge m_\lambda}2\theta_s {T_{s,2}\sum_{k \in A_{\lambda}}\ind{ k \in B_{s}}\ind{\cE_{k}}} 
%
	\end{align*} 
	where in the second sum we exchanged the sums over the arms and over the contribution of each epoch, 
	used the fact that for the arm $k$ active in the epoch $s$ we have that $s \leq m_{k}$ and thus we pay a regret of order at most $ 2\theta_s$ by pulling this arm. 

	%

%
	  For the second sum we observe that the sequence $\paren{\theta_{s}}^{1-\frac{1}{\alpha}}\paren{\log\paren{\cA T \theta_{s}^2}}^{\frac{1}{2\alpha}}$ is monotonically increasing for all $s \leq s_{\text{ end}}$ with ratio at least $\frac{6}{5}\paren{\sqrt{\frac{12}{5}}}^{\frac{1}{\alpha}-2}$ and that $$\sum_{k \in A_{\lambda}}\ind{k \in B_s} \ind{\cE_{k}} \leq b_s.$$ Therefore, we can write: 
	  \begin{align*}
			  \sum_{s=0}^{s_{\text{end}}\wedge m_{\lambda}} 2 \theta_{s}T_{s,2}\sum_{k \in A_{\lambda}}\ind{ k \in B_{s}}\ind{\cE_{k}} & \leq 4  \sum_{s=0}^{s_{\text{end}}\wedge m_{\lambda}}  \theta_{s} \paren{\frac{c_3 \log\paren{\mathcal{A}T\theta_{s}^{2}}}{\theta_{s}^{2}}}^{\frac{1}{2\alpha}} \frac{1}{b_{s}}\sum_{k \in A_{\lambda}}\ind{ k \in B_{s}}\ind{\cE_{k}} \\ 
			  & \leq 4 \paren{c_{3}}^{\frac{1}{2\alpha}} \sum_{s=0}^{s_{\text{ end}} \wedge m_{\lambda} }\theta_{s}^{1-\frac{1}{\alpha}}\paren{{ \log\paren{\mathcal{A}T\theta_{s}^{2}}}}^{\frac{1}{2\alpha}} \\
			  & \leq 4 \paren{c_{3}}^{\frac{1}{2\alpha}} c_{4}\theta_{ s_{ \text{ end}} \wedge m_{\lambda}+1}^{1-\frac{1}{\alpha}} \paren{ \log\paren{\cA T \theta^2_{s_{\text{ end}}\wedge m_{\lambda} +1 }}} \\
			  & \leq 2^{\frac{1}{\alpha}+1} c_4c^{\frac{1}{2\alpha}}_{3} \paren{\frac{\Delta_{\star, \lambda}}{4}}^{1-\frac{1}{\alpha}} \paren{ \log\paren{\frac{\cA T}{4} \Delta_{\star, \lambda}^2}},
	  \end{align*}
	  where we set $c_{4}=\frac{1}{1.2 *\sqrt{2.4}^{\frac{1}{\alpha}-2}-1}$, used the definition of $T_{s,2}$ from Algorithm~\ref{alg:cmix_ucb} and that the contribution of the regret of active arm $k \in A_{\lambda}$ is at most $\theta_{m_{\lambda}} \leq \frac{\Delta_{\star,\lambda}}{4}$ at the end. 
	  For the first sum, by plugging in the expression for $T_{s,1} \geq 1$ and using the assumption that in the epoch $s$ we sum up  the contributions of the regret of the arms which have not been eliminated until round $m_{k}$ we get: 
	  
	  \begin{align*}
		 \sum_{k \in A_{\lambda}}\Delta_{{k}}\sum_{s=0}^{s_{\text{ end}\wedge m_{\lambda}}}T_{s,1}\ind{\cE_{k}}\ind{k \in B_s} &\leq 2  \sum_{k \in A_{\lambda}} \Delta_{{k}}\sum_{s=0}^{s_{\text{ end}\wedge m_{\lambda}}} \frac{32\log\paren{\mathcal{A}T\theta_{s}^{2}}}{\theta_{s}^{2}} \ind{\cE_{k}}\ind{k \in B_s} \\
		 & \leq 64 \sum_{ k \in A_{\lambda}} \Delta_{{k}}\sum_{s=0}^{s_{ \text{ end}  }\wedge m_{\lambda} }\theta_s^{-2}\log\paren{ AT \theta_{s}^2} \ind{\cE_{k}}\ind{ k \in B_{s}} \\
		 & \leq 256 \sum_{k \in A_{\lambda}}\Delta_{{k}} \theta_{m_{k}}^{-2}\log\paren{\cA T \theta_{m_{k}}^2} \leq 1024 \sum_{k \in A_{\lambda}}\Delta_{{k}}^{-1}\log\paren{\cA T \Delta_{{k}}^2},
	  \end{align*}
	  where in the last line we used the geometrical increase of the series $\theta_{s}\log\paren{\cA T \theta^2_{s}}$ and the relation \eqref{eq:sup_bounds} between $\theta_{m_{k}}$ and $\Delta_{{k}}$. 

	Summing up all the terms we have the following upper bound: 
	\begin{align}
	\label{eq:pulls}
	\sum_{k \in A_\lambda} \e{N_k\ind{\cE_k}} \Delta_k \leq  \sum_{k \in A_{\lambda}} {
		 1024\Delta_{k}^{-1}\log{\paren{\mathcal{A}T\Delta_{k}^{2}}}} + 2^{-\frac{1}{\alpha}+3} c_4c^{\frac{1}{2\alpha}}_{3} \paren{{\Delta_{\star,\lambda}}}^{1-\frac{1}{\alpha}}\paren{{ \log\paren{\frac{\mathcal{A}T\Delta_{\star,\lambda}^{2}}{4}}}}^{\frac{1}{2\alpha}}.
	\end{align}
	Summing up the individual contributions of inequalities \eqref{eq:pulls} and \eqref{eq:late_elim} we obtain: 
	\begin{align*}
	R\paren{T} & \leq  
	2 \sum_{k \in A_{\lambda}} \paren{\Delta_{k} + \frac{96}{\mathcal{A}} \frac{1}{\Delta_{k}} + 512\Delta_{k}^{-1}\log{\paren{\mathcal{A}T\Delta_{k}^{2}}}} \\ 
	&\qquad+ \underbrace{2^{-\frac{1}{\alpha}+3} c_4c^{\frac{1}{2\alpha}}_{3}}_{ \tilde{c}} \paren{{\Delta_{\star,\lambda}}}^{1-\frac{1}{\alpha}}\paren{{ \log\paren{\frac{\mathcal{A}T\Delta_{\star,\lambda}^{2}}{4}}}}^{\frac{1}{2\alpha}} + \frac{64}{\mathcal{A}} \sum_{ k \in A_{0}\setminus A_{\lambda}} \frac{1}{\lambda}  + \lambda T. 
	\end{align*} 
	Finally, by noticing that $\Delta_{{k}} \leq 1$  we imply the statement of the Theorem. 
\end{proof}

\subsection{  Proof of Theorem \ref{thm:prob_indep}}
\begin{proof}
	We provide the main argument while disregarding the influence of absolute multiplicative constants. Firstly, for any choice $\lambda > \frac{1}{2}\sqrt{\frac{e^{1/2}}{T}}$ we have that $\sum_{ k \in A_{\lambda}} \frac{\log\paren{\cA T \Delta^{2}_{k}}}{\Delta_{k}^{}} + \sum_{ k \in A_{0}\setminus A_{\lambda}} \frac{1}{\lambda} < K\frac{\log T}{\lambda}$. Furthermore, from the definition of $\Delta_{\star,\lambda}$, since $1-\frac{1}{\alpha} < 0$ we have: 
	\begin{align*}
	\Delta_{\star,\lambda}^{1-\frac{1}{\alpha}} \paren{\log \paren{\cA T \Delta_{\star}^{2}}}^{\frac{1}{2\alpha}} < {\lambda}^{1 - \frac{1}{\alpha}} \paren{\log T}^{\frac{1}{2\alpha}}.
	\end{align*} Thus, we obtain the following worst case bound: 
	\begin{align}
	\label{eq:regret_ind}
	R\paren{T} \leq K\frac{\log\paren{\mathcal{A}T}}{\lambda} + \lambda^{\frac{\alpha -1}{\alpha}} \paren{\log\paren{T}}^{\frac{1}{2\alpha}} + \lambda T.
	\end{align}
	Now consider two different scenarios. If $K < T^{1-2\alpha}$ then, as one can readily check,  by setting in this case $\lambda = T^{-\alpha}\log^{\frac{1}{2}}\paren{T}$ (which is an admissible choice according to the Theorem~\ref{thm: upp_bound} we obtain : 
	\begin{align*}
	R\paren{T} \leq C_{1} T^{1-\alpha} \paren{\log T}^{\frac{1}{2\alpha}},
	\end{align*}
	where $C_{1}$ is some numerical constant.
	
	Otherwise, (i.e. if $K > T^{1-2\alpha}$) by setting $\lambda = \sqrt{\frac{K}{T}}$ we get the following bound: 
	\begin{align*}
	R\paren{T} \leq C_{2} \sqrt{TK}\log \paren{T},
	\end{align*}
	as is this case the second term from Equation~\eqref{eq:regret_ind} dominates the bound.
	Uniting these results into one and taking $C_{3} = \max \{C_{1},C_{2}\}$ we obtain the claim of the Theorem.
	
\end{proof}

		\subsection{Proof of Proposition~\ref{thm:prob_ind_lb}}
		Without losing of generality, we suppouse that random rewards are bounded in $[-1,1]$. Recall, that we use notation $[K] = \{1,\ldots,K\}$.
		For  $0 \leq i \leq K$ we construct the stochastic bandit environment $\cB_{i}$ in the following way. 
		
		For the arm $a \in [K]$ in bandit $\cB_i$ set the distribution of rewards $\nu_{i}^{a}$ as 
		$$ \nu_{i}^{a}= m_{0} \paren{\cR\paren{\frac{1}{2}}\mbi_{a\neq i} + \cR\paren{\frac{1}{2}+\epsilon}\mbi_{a = i}},$$
		 where $\cR(p)$ is Rademacher distribution with parameter $p$, i.e. $\cR\paren{p} = (1-p)\delta_{-1} + p\delta_{1}$; $\epsilon = 1/8$ and $m_{0}$ set to be $T^{-\alpha}$. 	In every bandit $i$ for each arm $a \in [K]$ we assume that the sample rewards are "frozen" from the beginning and drawn from the distribution $\nu_{i}^{a}$. More precisely, for the process $\paren{X_{t}^{a}}$ attached to the arm $a$ in Bandit $\cB_i$ we define  $X_{t+\ell}^{a}\paren{\omega} = X_{t}^{a}\paren{\omega} \sim \nu_{i}^{a}$ for every $\ell \leq T$, $i \in \{0,\ldots,K\}, a \in [K]$. One can readily check that process $X_{t}^{a}$ satisfies Definition~\ref{def:phi_C_mixing} with rate $\Phi_{\cC}(t) = 2t^{-\alpha}$. Furthermore, for a sample $X_{t}^{a} \sim \nu_{i}^a$ from arm $a$ in Bandit $\cB_i$ we have $\ee{}{X_{t}^{a}} = 2\epsilon m_{0} = \frac{1}{4} T^{-\alpha}$ if $a=i$, and $\ee{}{X_{t}^{a}} =0$ otherwise.

For the arm $a $ define the following event: 
\begin{align*}
E_{a} = \{ N_a \leq cT \},
\end{align*} 
where $c>0$ is some small universal constant and recall that $N_a$ is the number of pulls of arm $a$ until time-horizon $T$. 
We have that under measure $\mbp_{\cB_0}$: 
\begin{align*}
T = \ee{}{\sum_{a =1}^{K} N_a} \geq \sum_{a=1}^{K}\ee{}{N_a|E_{a}^{c}}\probb{\cB_{0}}{E_{a}^{c}} \geq cTK \min_{a \in [K]}\probb{\cB_{0}}{E_{a}^{c}},
\end{align*}		
which after considering the complementary event implies that 
\begin{align*}
\max_{a \in [K]} \probb{\cB_{0}}{E_a} \geq 1 -\frac{1}{cK} \geq 1- \frac{1}{2c},
\end{align*}
since $K \geq 2$.
Let $a_{0} = \argmax_{a \in [K]}\probb{\cB_{0}}{E_{a}}$ be any element that achieves the previous maximum.  For the event $E_{a_{0}}$ in bandit $\cB_{a_0}$ by using change of measure principle between two Rademacher distributions we have: 
\begin{align*}
\probb{\cB_{a_0}}{E_{a_0}} & = \ee{\cB_{0}}{\mbi_{E_{a_{0}}} \exp\paren{ \frac{X_{t}}{2m_{0}}\log\paren{ \frac{1+2\epsilon}{1-2\epsilon}} + \frac{1}{2}\log\paren{ \frac{1+2\epsilon}{1-2\epsilon}} }} \\
 & \geq \ee{\cB_{0}}{ \mbi_{E_{a_{0}}} \exp\paren{ \frac{-m_{0}}{2m_{0}}\log\paren{ \frac{1+2\epsilon}{1-2\epsilon}} + \frac{1}{2}\log\paren{ \frac{1+2\epsilon}{1-2\epsilon}} }} \\
 & = \probb{\cB_{0}}{E_{a_{0}}} \geq 1- \frac{1}{2c}.
\end{align*} 
Therefore, for the regret under bandit $\cB_{a_0}$ we get:
\begin{align*}
\ee{\cB_{a_0}}{R\paren{T}} &\geq \ee{\cB_{a_0}}{R\paren{T}|E_{a_{0}}} \probb{\cB_{a_0}}{E_{a_{0}}} \geq T(1-c)2\epsilon m_{0} \probb{\cB_{a_0}}{E_{a_{0}}} \\
& \geq \frac{1-c}{4}\paren{1-\frac{1}{2c}}T^{1-\alpha} \geq \frac{1}{80}T^{1-\alpha},
\end{align*} 
which implies the bound on the minimax regret.

\end{document}